\documentclass[sigconf]{acmart}
\usepackage{amsmath,amsfonts}
\usepackage{algorithmic}
\usepackage{hyperref}
\usepackage{url}
\usepackage{algorithm}
\usepackage{xcolor}
\usepackage{mwe}
\newlength{\tempheight}
\newlength{\tempwidth}

\newcommand{\rowname}[1]% #1 = text
{\rotatebox{90}{\makebox[\tempheight][c]{\textbf{#1}}}}
\newcommand{\columnname}[1]% #1 = text
{\makebox[\tempwidth][c]{\textbf{#1}}}
\usepackage{listings}
\usepackage{xparse}
\NewDocumentCommand{\codeword}{v}{%
\texttt{\textcolor{blue}{#1}}%
}
\usepackage{nicefrac}
\usepackage{amsthm}
\usepackage{graphicx}
\usepackage{subcaption}
\newtheorem{theorem}{Theorem}
\newtheorem{lemma}{Lemma}

\newtheorem{definition}{Definition}
\usepackage{graphicx}
\usepackage{textcomp}
\usepackage{comment}
\usepackage{pifont}% http://ctan.org/pkg/pifont
\newcommand{\cmark}{\ding{51}}%
\newcommand{\xmark}{\ding{55}}%
\AtBeginDocument{%
  \providecommand\BibTeX{{%
    \normalfont B\kern-0.5em{\scshape i\kern-0.25em b}\kern-0.8em\TeX}}}

\copyrightyear{2022}
\acmYear{2022}
\setcopyright{acmcopyright}
\acmConference[WSDM '22] {Proceedings of the Fifteenth ACM International Conference on Web Search and Data Mining}{February 21--25, 2022}{Tempe, AZ, USA.}
\acmBooktitle{Proceedings of the Fifteenth ACM International Conference on Web Search and Data Mining (WSDM '22), February 21--25, 2022, Tempe, AZ, USA}
\acmPrice{15.00}
\acmISBN{978-1-4503-9132-0/22/02}
\acmDOI{10.1145/3488560.3498498}
\begin{document}

\title{Differentially Private Ensemble Classifiers for Data Streams}

\author{Lovedeep Gondara}
\affiliation{%
\department{School of Computing Science}
  \institution{Simon Fraser University}
  \country{British Columbia, Canada}
}
\email{lgondara@sfu.ca}

\author{Ke Wang}
\affiliation{%
\department{School of Computing Science}
  \institution{Simon Fraser University}
  \country{British Columbia, Canada}
}
\email{wangk@cs.sfu.ca}

\author{Ricardo Silva Carvalho}
\affiliation{%
\department{School of Computing Science}
  \institution{Simon Fraser University}
  \country{British Columbia, Canada}
}
\email{rsilvaca@sfu.ca}

\renewcommand{\shortauthors}{Gondara, et al.}

\begin{abstract}
Learning from continuous data streams via classification/regression is prevalent in many domains. Adapting to evolving data characteristics (concept drift) while protecting data owners' private information is an open challenge. We present a differentially private ensemble solution to this problem with two distinguishing features: it allows an \textit{unbounded} number of ensemble updates to deal with the potentially never-ending data 
streams under a fixed privacy budget, and it is \textit{model agnostic}, in that it treats any pre-trained differentially private classification/regression model as a black-box. Our method outperforms competitors on real-world and simulated datasets for 
varying settings of privacy, concept drift, and data distribution. 
\end{abstract}





%\settopmatter{printfolios=true}
\maketitle

\section{Introduction} \label{sec:intro}
Continuous data streams generate large volumes of data, with examples being data from wearables \cite{de2016iot}, biosensors in medicine \cite{ebada2020applying}, social media \cite{lukes2018sentiment}, news \cite{ksieniewicz2020fake}, mobile applications \cite{Allix:2016:ACM:2901739.2903508}, electronic health records \cite{cihidad},
credit card transactional flows \cite{dal2015credit}, malware data \cite{Allix:2016:ACM:2901739.2903508}. To assist in decision making, machine learning models need to handle data streams efficiently. Scalability is not the only challenge; we have to consider that properties and patterns of data are subject to change over time, a phenomenon known as \emph{concept drift}. For example,  malware files and fake news evolve over time to evade detection \cite{Allix:2016:ACM:2901739.2903508,ksieniewicz2020fake}.

To further add to the challenge, data streams from many domains involve sensitive, personal information about contributing users, such as patients' records and user data in mobile applications, protection of which is of paramount interest. While concept drift and privacy have been extensively studied in isolation, works considering both are in infancy. See more discussion in Section \ref{sec:related}. 
In this work, our goal is to allow machine learning models to deal with concept drift when training on potentially never-ending data streams involving sensitive data, where the model(s) learned can be published without disclosing sensitive information. 
To that end, we consider Differential Privacy (DP) \cite{Dwork:2006:CNS:2180286.2180305,abadi2016deep,shokri2015privacy} as the privacy definition, and
widely used ensemble learning for data streams  \cite{cano2020kappa} as the modelling paradigm.

\subsection{Challenges}\label{sec:challenge}
Enforcing privacy on ensembles handling concept drift is not trivial. The main approaches of ensembles over data streams \cite{cano2020kappa}, such as weight modification \cite{kolter2007dynamic,cano2020kappa} and online ensemble update \cite{pietruczuk2017adjust,cano2020kappa} are not ideal for the privacy-preserving scenario where with new incoming instances, the former continuously measures the performance of a diverse set of classifiers to update the weights and the latter continuously updates the pool of online models. This \emph{continuous} update could lead to privacy budget depletion due to composition of privacy loss, limiting the number of updates before the privacy budget runs out. Our goal is to deal with the never-ending data streams by allowing for an \emph{unbounded} number of updates under a \textit{fixed} privacy budget.

\subsection{Our Proposals}
We propose a DP temporal ensemble approach in the form of \emph{dynamic ensemble line-up} \cite{wang2003mining, cano2020kappa}. In the non-private setting, the data stream comes in the form of labeled data chunks $D_1,\cdots, D_t$ until the current time $t$, and at the time $t$ the ensemble of size $k$, denoted $\mathcal{E}_t$, consists of $k$ models $M_{\tau(1)},\cdots,M_{\tau(k)}$ trained on the $k$ corresponding data chunks $D_{\tau(1)},\cdots,D_{\tau(k)}$.
$\tau(i)$ maps the (relative) position $i$ within $\mathcal{E}_t$ to the corresponding (absolute) time point 
in the data stream. 
With a weighting scheme, these $k$ models collectively make the prediction for unlabeled data at the next time $t+1$. As the labeled data chunk $D_{t+1}$ becomes available, a new model is trained on $D_{t+1}$ and 
the next ensemble $\mathcal{E}_{t+1}$ is obtained by replacing either the oldest model or the weakest model in  $\mathcal{E}_{t}$ with the new model. By replacing (instead of updating) some existing model, this approach is particularly suited for limiting the accumulation of privacy loss.  Our contributions are specifically described as follows. 

\begin{enumerate}

\item (Section \ref{sec:prob_statement}) We formulate the problem of DP temporal ensembles for a release history with the formal definition provided in Section \ref{sec:dp_ens}.

\item (Section \ref{sec:dp_ens_weight}) At the core of our DP ensemble mechanism is the DP weighting scheme for aggregating the prediction of component models. We present the DP weighting scheme for classification and regression. Our method is \textit{model agnostic}, that is, it treats DP models $M_i$ as black-boxes. 

\item  (Section \ref{sec:dp_ens}) We present a DP ensemble mechanism for releasing the ensemble $\mathcal{E}_t=\{M_{\tau(1)},\cdots,M_{\tau(k)}\}$ at any time $t$, to ensure that the DP guarantee holds even if the adversary has access to all released ensembles $\mathcal{E}_{t'}$ for $t'\leq t$. Our proposal
allows an \emph{unlimited} number of ensemble updates for never ending data streams at a constant privacy budget. 

\item (Section \ref{sec:DP-transfer}) 
To demonstrate the benefits of our method for deep neural network models where the potentially large number of model parameters present a challenge for retaining utility under DP guarantee, we consider a transfer learning option 
for boosting utility where a public dataset is available.

% whenever possible.  

\item  (Section \ref{sec:DP-transfer}) We provide empirical evidence on the effectiveness of the proposed DP temporal ensemble using real-world and simulated datasets. The source code and datasets will be made publicly available for reproducibility\footnote{\url{https://github.com/lgondara/DPTemporalEnsemble}}.
\end{enumerate}

% Rest of the paper is organized as follows. Section \ref{sec:related} surveys related work. Section \ref{sec:prelim} provides introduction to DP. Section \ref{sec:prob_statement} formally introduces the problem. Then in Section \ref{sec:dp_ens_weight} we propose our novel DP weight mechanisms for classification and regression. Section \ref{sec:dp_ens} introduces our DP temporal ensemble and provides it's DP guarantees followed by detailed evaluation in Section \ref{sec:DP-transfer}. We conclude the paper in Section \ref{sec:conclusion}.

\section{Related Work}\label{sec:related}
Non-private temporal ensembles have been studied before \cite{wang2003mining,scholz2007boosting,kolter2007dynamic,cano2020kappa}. See \cite{gama2014survey, lu2018cdreview, cano2020kappa} for a review. These methods can be classified into either explicit or implicit. \textit{Explicit} methods use drift detection and only update the model when drift is detected. Examples are \cite{de2018ddmfishers} and \cite{yu2017ddmhierarchical}. \textit{Implicit} methods do not detect drift but adapt the model to account for changes automatically. The updates can be incremental in a single classifier \cite{elwell2011incremental}, or weighted in an ensemble \cite{wang2003mining}. Ensemble models in the implicit setting usually outperform other approaches \cite{lu2018cdreview, wang2003mining}. Our work 
is an implicit method and adapts dynamic ensemble line-up \cite{cano2020kappa} but deals with sensitive data.

For privacy preserving works on data streams, \cite{zhang2019bayesdp} proposes DP Bayesian classifiers with explicit drift detection. This method continuously updates the model parameters using incoming sensitive data which reduces the privacy budget by some amount after each update, thus, the privacy budget will run out after a finite number of updates. Also, the method does not account for the privacy loss for drift detection and the privacy loss of updates when no concept drift is detected. The works in \cite{fang2020localdp,khavkin2018clusterdp,fanaeepour2019privstream,dwork2010pan,kellaris2014differentially} focused on releasing summary statistics such as counts, mean, mode, range queries, centroids, etc. 

% For works on concept drift with privacy preservation, \cite{zhang2019bayesdp} proposes DP Bayesian classification with explicit drift detection. However, privacy for a data chunk is only taken into account when a concept drift is detected, so the privacy loss for sensitive chunks without a detected drift is not accounted for. Also, this method does not accommodate unlimited DP model updates. 

There are previous works on \emph{static} differentially private ensembles such as \cite{jagannathan2009practical,xiang2018collaborative}. These works are not designed to handle streaming data because they cannot accommodate concept drift.

%%%%%%%%%%%%%%%%%%%%%%%%%%%%%%%%%%%%%%%%%%%%%%%%%%%%%%%%%%%%%%%%%%%%%

\section{Differential Privacy}\label{sec:prelim}

\begin{definition}[Neighbors and Sensitivity]\label{def:sensitivity}
Two data sets $D$ and $D'$ are \textit{neighboring} if they differ due to the substitution of exactly one sample. The \textit{sensitivity} of a function $f: \mathbb{D}^n \rightarrow \mathbb{R}^d$, denoted by $\Delta f$, is $max_{D,D'} ||f(D)-f(D')||$ over all neighboring pairs $D$ and $D'$. 
% For a given data $D$, the local sensitivity of $f$ wrt $D$
% is defined as $max_{D'} |f(D)-f(D')|$ over all $D'$ such that $D$ and $D'$ are neighboring.  
\end{definition}

\begin{definition}[Differential Privacy \cite{dwork2006our}]\label{def:dp}
A randomized mechanism $\mathcal{M}: \mathbb{D}^n \rightarrow \mathbb{R}^d$ is $(\varepsilon,\delta)$-differentially private if for any pair of neighbouring data sets $D,D' \in \mathbb{D}^n$, and for all sets $S$ of possible outputs:
\begin{equation}
    \text{Pr}[\mathcal{M}(D) \in S] \le e^{\varepsilon} \text{Pr}[\mathcal{M}(D') \in S] + \delta 
\end{equation}
\end{definition}

Since neighboring datasets differ by the data of one user, the inequality above ensures that the output of a mechanism $\mathcal{M}$ satisfying DP will have a \emph{small} impact, through the multiplicative factor  $e^{\varepsilon}$ and the additive factor $\delta$,
if we remove or add any single user from the dataset used to generate the output. 
% If $\delta=0$, we have pure $\varepsilon$-differential privacy.

\begin{theorem}[Parallel Composition \cite{frank2009}]\label{thm:pc}
Let $\mathcal{M}_i$ each provide $(\varepsilon_i,\delta_i)$-differential privacy.
Let $D_i$ be arbitrary disjoint subsets of $D$.
The sequence of $\mathcal{M}_i(D_i)$ provides $(max_i \varepsilon_i, max_i \delta_i)$-differential privacy.
\end{theorem}

\begin{theorem}
[Sequential Composition \cite{frank2009}\cite{dwork2014book}]\label{thm:sc}
Let $\mathcal{M}_i$ be $(\varepsilon_i,\delta_i)$-differentially private.
The adaptive sequence of $\mathcal{M}_i$ is $(\sum_i \varepsilon_i, \\ \sum_i \delta_i)$-differentially private. 
\end{theorem}

\begin{theorem}[Post Processing \cite{dwork2014book}] 
\label{thm:post}
Let $\mathcal{M}: \mathbb{D}^n \rightarrow \mathbb{R}^d$ be a randomized mechanism that is ($\varepsilon, \delta$)-differentially private. Let $f:\mathbb{R}^d \rightarrow \mathbb{R}^t$ be a deterministic function. Then $f \circ \mathcal{M}: D^n \rightarrow \mathbb{R}^t$ is ($\varepsilon, \delta$)-differentially private.
\end{theorem}

\begin{definition}
[Laplace Mechanism \cite{Dwork:2006:CNS:2180286.2180305}]
Given any function $f: \mathbb{D}^n \rightarrow \mathbb{R}^d$, the Laplace mechanism is defined as: 
$\mathcal{M}(D, f(·), \varepsilon) = f(D) + (Y_1,\cdots, Y_d)$,
where $Y_i$ are i.i.d. random variables drawn from $Lap(\Delta f/\varepsilon)$.
\end{definition}

\begin{theorem}[DP of Laplace Mechanism \cite{Dwork:2006:CNS:2180286.2180305}]\label{thm:lap}
The Laplace mechanism is $(\varepsilon, 0)$-differentially
private.
\end{theorem}

%%%%%%%%%%%%%%%%%%%%%%%%%%%%%%%%%%%%%%%%%%%%%%%%%%%%%%%%%%%%%%%%%%%
\section{Problem Statement}\label{sec:prob_statement}
We consider a data stream $D_1,\cdots,D_t$ where each $D_i$ is a chunk of data generated at the time $i$ and $t$ is the \textit{current time}. Samples in $D_i$ are labeled with a class variable and we are interested in using the data $D_1,\cdots,D_t$ to predict the class for unlabeled samples at the next time $t+1$. A new data chunk $D_{t+1}$ that becomes available at the next time $t+1$ is added to the stream, so the stream is potentially unbounded.
We consider the dynamic setting where the data stream is susceptible to concept drift.

An ensemble of size $k$ at time $t$, $\mathcal{E}_t =\{M_{\tau(1)},\cdots,M_{\tau(k)}, w_{\tau(1)},\\
\cdots,w_{\tau(k)}\}$, consists of $k$ prediction models $M_i$s respectively trained on $D_{\tau(1)},\cdots,D_{\tau(k)}$, and their weights $w_i$, where $\tau(i)$ denotes the (absolute) time point corresponding to the (relative) position $i$ within $\mathcal{E}_t$. 
We assume $\tau(1),\cdots,\tau(k)$ are listed in the ascending order. 

For \textit{classification}, we have a categorical class variable $c$, $M_{i,c}(x)\in [0,1]$ denotes the prediction score by $M_i$ for $c$ on a sample $x$, and the overall score predicted by $\mathcal{E}_t$ is computed by 
\begin{equation}\label{eqn:ec}
    \mathcal{E}_{t,c} (x) = \frac{\sum_{i=1}^{k} w_{\tau(i)} \cdot M_{\tau(i),c}(x)}{\sum_{i=1}^{k} w_{\tau(i)}}
\end{equation}
The predicted class for $x$ is the class $c$ with the maximum $\mathcal{E}_{t,c}(x)$. 

For \textit{regression}, we have 
a continuous class variable, normalized to within the range $[0,1]$, $M_i(x)\in  [0,1]$ denotes the predicted value by $M_i$ for $x$ and the overall predicted value by $\mathcal{E}_t$ is given by
\begin{equation}\label{eqn:en}
    \mathcal{E}_t (x) = \frac{\sum_{i=1}^{k} w_{\tau(i)} \cdot M_{\tau(i)}(x)}{\sum_{i=1}^{k} w_{\tau(i)}}
\end{equation}

As $D_{t+1}$ becomes available at time $t+1$, we use it to train a new model and re-estimate the weights for all models in the ensemble (more details later).  Then we update $\mathcal{E}_t$ to $\mathcal{E}_{t+1}$ by replacing either the \textit{oldest} model or the \textit{worst} (i.e., with smallest weight) model in $\mathcal{E}_t$.

\textbf{Problem of $(\varepsilon,\delta)$-DP Temporal Ensembles.} 
The focus of this work is the scenario where each $D_i$ contains sensitive, private information about the contributing users. We assume that each sample in $\cup_i D_i$ has a unique identifier and all samples are independently generated. This independence assumption would allow us to treat all $D_i$ as disjoint subsets of $\cup D_i$. 
We want to ensure that, at any current time $t$, the entire history of released ensembles up to $t$, i.e., $\mathcal{E}_t'$ for all $t'< t$, must satisfy $(\varepsilon,\delta)$-DP for given $\varepsilon,\delta$. 
A formal definition of $(\varepsilon,\delta)$-DP 
for an ensemble and for a history will be given in Section \ref{sec:dp_ens_snp} and in Section \ref{sec:dp_t_time}.

We develop noisy weight mechanisms in Section~\ref{sec:dp_ens_weight} and provide
the privacy analysis for the overall temporal ensemble mechanism in Section \ref{sec:dp_ens}. 
The given $\varepsilon$ is split into $\varepsilon_1$ 
for training models $M_i$ and $\varepsilon_2$
for computing noisy weights $w_i^*$, where $\varepsilon = \varepsilon_1 + \varepsilon_2$.

%%%%%%%%%%%%%%%%%%%%%%%%%%%%%%%%%%%%%%%%%%%%%%%%%%%%%%%%%%%%%%%

\section{Noisy Weight Estimation}\label{sec:dp_ens_weight}
We assume that the labeled data $D_i$ is split into training, validation, and testing subsets. For an ensemble $\mathcal{E}_t$ at time $t$, the weights $w_{\tau(1)},\cdots,w_{\tau(k)}$ are measured using the performance on the validation subset of $D_{t}$, denoted by $V_t$. We choose the validation subset of $D_{t}$ to calculate the weights for all models in  $\mathcal{E}_t$ 
because $t$ is closest to the next time point $t+1$ that the ensemble at time $t$ aims to predict. In this section, we present the noisy estimation of $w_{\tau(1)},\cdots,w_{\tau(k)}$ and we present the privacy analysis in Section  \ref{sec:dp_ens}. 

For classification, we consider two settings. In the \textbf{general setting}, we measure the classification accuracy for all classes. In the \textbf{focused setting}, we consider the accuracy of a chosen class called the positive class, which is commonly used in class-imbalanced problems such as fake news/malware/disease detection.

\subsection{Classification - General Setting}\label{sec:weight_gen_class}
Consider a model $M_i$ in $\mathcal{E}_t$ and the validation subset $V_t$. 
In the general setting, we consider the classification error of $M_i$ defined by the \textbf{Mean Squared Error} (MSE), as in  \cite{wang2003mining,brzezinski2011accuracy}:
\begin{equation}\label{eqn:mse}
    MSE_i = \frac{1}{|V_t|} Err_{i}
\end{equation}
where
\begin{equation}\label{eqn:mse_u}
    Err_{i} = \sum_{x \in V_t} (1-M_{i,c}(x))^2
\end{equation}
and $M_{i,c}(x)$ is the score given by $M_i$ for the instance $x$ and its true class $c$. 
For a random predictor, the $MSE$ is given by 
\begin{equation}
 MSE_{r} = \sum_c p(c)(1-p(c))^2
\end{equation}
 where $p(c)$ is the proportion of class $c$ in $V_t$. 
 The weight $w_i$ for $M_i$ is defined as the hinge loss:
 
 \begin{equation}\label{eqn:mse_2}
     w_i = \text{max}(0,MSE_r - MSE_i)
 \end{equation}
%  A larger $w_i$ is given to a more accurate $M_i$.

\begin{lemma}[General setting]
Let $\Delta Err_i$ denote the sensitivity of $Err_i$ defined over all neighboring validation subsets $V_t,V'_t$. $\Delta Err_i=1$.
\end{lemma}
\begin{proof}
Consider computing Equation \eqref{eqn:mse_u} for 
neighboring validation sets $V_t,V'_t$. All $M_{i,c}(x)$ are same except for one instance $x$, so $Err_i$ differs by at most 1 because $M_{i,c}(x)$ is at most 1. 
\end{proof}

\textbf{Computing Noisy Weight $w_i^*$}: $w_i^*$ is computed from Equation \eqref{eqn:mse_2} and  \eqref{eqn:mse} using  the noisy $Err^*_i$:
\begin{equation}\label{eqn:dp_cat}
    Err^*_i = Err_i+Lap(\nicefrac{\Delta Err_i}{\varepsilon_2})
\end{equation}

\subsection{Classification - Focused Setting}\label{sec:weight_foc_class}
The focused setting is concerned with prediction performance of the positive class. Typically the positive class has a small proportion compared to other classes called the negative class and the general classification accuracy does not reflect the accuracy of the positive class. In this case, we consider the \textbf{balanced accuracy} (BA):
\begin{equation}\label{eqn:BA}
    BA = a_1\cdot TPR + a_2\cdot TNR
\end{equation}
where $a_1$ and $a_2$ are constants and $a_1+a_2=1$. $TPR$ and $TNR$ are the true positive rate (the proportion of predicted positives that are actually positive) 
and the true negative rate (the proportion of predicted negatives that are actually negative). BA is in the range $[0,1]$. Since $TNR=1-FPR$, where $FPR$ is the false positive rate (the proportion of negatives that are predicted as positives), BA is related to (TPR,FPR) commonly used for generating AUC. The above BA generalizes the  balanced accuracy in \cite{brodersen2010balanced} that assumes  $a_1=a_2=\nicefrac{1}{2}$.

To obtain the noisy weight, we assume that some estimates of the proportions of positive samples and negative samples in $D_i$, denoted by $p$ and $n$ with $p+n=1$, are public. These are \emph{not} necessarily the exact proportions in the sensitive data, but rather are estimates from general knowledge (for example, $p$ and $n$ come from the general knowledge about the entire data stream). These estimates allow us to estimate $BA_i$ for $M_i$ as follows:

\begin{equation}
    BA_i = a_1\frac{TP_i}{p \cdot |V_t|} + a_2\frac{TN_i}{n \cdot |V_t|} \label{eq:bai_v2}
\end{equation}
where $TP_i$ (resp. $TN_i$) is the number of positive instances (resp. negative instances) in $V_t$ that are predicted by $M_i$ as positive (negative). 

\begin{lemma}[Focused Setting]
Let $\Delta BA_i$ denote the sensitivity of $BA_i$ defined over neighboring pairs $V_t,V'_t$. $\Delta BA_i = \frac{1}{|V_t|} max ( \frac{a_1}{p}, \frac{a_2}{1-p} )$.
\end{lemma}
\begin{proof}
Consider neighboring validation subsets $V_t,V'_t$. 
% Recall the size of the data chunks is public, thus $|V|$ is known. 
For simplicity, we drop the index $i$ below.
\begin{gather}
    \nonumber BA-BA' = \frac{a_1}{p} \times (\frac{TP}{|V_t|} - \frac{TP'}{|V_t|}) + \frac{a_2}{n} \times (\frac{TN}{|V_t|} - \frac{TN'}{|V_t|})
\\
    \nonumber = \frac{1}{|V_t|} \bigg( \frac{a_1}{p} \times (TP - TP') + \frac{a_2}{n} \times (TN - TN') \bigg)
\end{gather}
For neighboring $V_t,V'_t$ where only one sample is different, there are four possible cases: (i) both $TP-TP'$ and $TN-TN'$ are 0, (ii) one of $|TP-TP'|$ and $|TN-TN'|$ is 1 and the other is 0, (iii)  $TP-TP'=1$ and $TN-TN'=-1$, (iv)
$TP-TP'=-1$ and $TN-TN'=1$. Noting $p+n=1$, we have:
\begin{gather}
    \nonumber |BA-BA'| \leq  \frac{1}{|V_t|} \times max ( \frac{a_1}{p}, \frac{a_2}{1-p} )
\end{gather}
\end{proof}

In the above lemma we assume that the validation size $|V_t|$ is public. The same assumption was made in \cite{barthe2016differentially, abadi2016deep,kamath2020primer} (e.g. see Remark~2.2 in \cite{kamath2020primer}). Alternatively, if a minimum validation size $|V_t|$ for all $V_t$s is required (for the purpose of statistical significance), we can use the minimum size in $\Delta BA_i$ without referring to specific $|V_t|$. 

\textbf{Computing Noisy Weight $w_i^*$}: We add the Laplace noise:
\begin{equation}\label{eqn:dp_w2}
    w^*_i = BA_i + Lap(\nicefrac{\Delta BA}{\varepsilon_2})
\end{equation}

\subsection{Regression}\label{sec:regression} 
For a continuous class variable, we define $Err_i$ as
\begin{equation}\label{eqn:mse_c}
    Err_{i} = \sum_{x \in V_t} (y_x-\hat{y}_{x})^2
\end{equation}
where $y_x$ is the true class value of $x$ and $\hat{y}_x$ is the predicted class value by the regression model $M_i$. We then get $MSE_i$ as
\begin{equation}\label{eqn:mse_reg}
    MSE_i = \frac{1}{|V_t|} Err_{i}
\end{equation}
and define the weight for $M_i$ as
\begin{equation}\label{eqn:mse_4}
    w_i = \frac{1}{MSE_i+\mu}
\end{equation}
$\mu$ is a small constant to allow weight calculation in rare situations when $MSE_i= 0$. $\mu = 10^{-5}$ is used in our experiments.

\begin{lemma}[Regression] 
Let $\Delta Err_i$ denote the sensitivity of $Err_i$ defined over all neighboring validation subsets $V_t,V'_t$. $\Delta Err_i=1$.
\end{lemma}
\begin{proof}
Recall that the true class value and the predicted class value are in the range $[0,1]$. So, for neighboring validation sets $V_t,V'_t$, $Err_i$ differs by at most 1 because  $(y_x-\hat{y}_x)^2$ is at most 1. 
\end{proof}

\textbf{Computing Noisy Weight $w_i^*$}: 
We add the Laplace noise \begin{equation}\label{eqn:dp_continuous}
    Err^*_i = Err_i+Lap(\nicefrac{\Delta Err_i}{\varepsilon_2})
\end{equation}
and compute $w^*_i$ using $Err^*_i$, Eqn. \eqref{eqn:mse_reg} and Eqn. \eqref{eqn:mse_4}.

\textbf{Discussion.}  The weighting scheme for both classification and regression is model agnostic, that is, it treats the DP models $M_i$ as black-boxes. This is because 
the computation of the weights $w_i$s only depends on the outputs, not the internal working of $M_i$. 

\section{DP Temporal Ensemble}\label{sec:dp_ens}
In Section \ref{sec:dp_ens_snp}, we provide the privacy analysis for a single $(\varepsilon,\delta)$-DP ensemble $\mathcal{E}_t =\{M_{\tau(1)},\cdots,M_{\tau(k)}, w_{\tau(1)}^*,\cdots,w_{\tau(k)}^*\}$, where each model $M_i$ is trained using any method on the training subset $S_i$ of $D_i$ and its weight $w_i^*$ is computed using the validation subset $V_{t}$ of $D_{t}$. In Section 
\ref{sec:dp_t_time}, we update $\mathcal{E}_t$ to $\mathcal{E}_{t+1}$ and present the privacy analysis for releasing all ensembles $\mathcal{E}_1,\cdots, \mathcal{E}_{t}$ up to the time $t$.

\subsection{Releasing A Single Ensemble}\label{sec:dp_ens_snp}
First, we extend the notion of DP in Definition \ref{def:sensitivity} to releasing an ensemble $\mathcal{E}_t$.
Let $X_t=<S_{\tau(1)},\cdots,S_{\tau(k)},V_{t}>$, where $S_{\tau(1)},\cdots,S_{\tau(k)}$ are the training subsets for  $M_{\tau(1)},\cdots,M_{\tau(k)}$ and $V_{t}$ is the validation subset of $D_t$ for computing the noisy weights  $w_{\tau(1)}^*,\cdots,w_{\tau(k)}^*$.

\begin{definition}[Neighboring Datasets for Ensembles]\label{def:neighbor2} Consider
$X_t=<S_{\tau(1)},\cdots,S_{\tau(k)},V_{t}>$ and $X'_t=<S'_{\tau(1)},\cdots,S'_{\tau(k)},V'_{t}>$. We say that $X_t$ and $X'_t$ are \textit{neighboring} if $\cup_i S_i \cup V_{t}$ and $\cup_i S'_i \cup V'_{t}$ (duplicates preserved) 
are neighboring in the sense  of Definition \ref{def:sensitivity}. 
\end{definition}

Note that $X_t$ and $X'_t$ are neighboring if and only if 
either for one $i$, $S_i$ and $S'_i$ are neighboring and $S_j=S'_j$ for all $j\neq i$, or $V_{t}$ and $V'_{t}$ are neighboring and $S_i=S'_i$ for all $i$. 

\begin{definition}[Differential Privacy for Ensembles]\label{def:DP-ensemble}
A mechanism $\mathcal{C}$ from the domain of $X_t$ to the range of $\mathcal{E}_t$ is $(\varepsilon,\delta)$-differentially private 
if for all neighbouring pairs $(X_t, X'_t)$ and for all sets $\mathcal{O}$ of possible outputs:
\begin{equation}
    \Pr[\mathcal{C}(X_t) \in \mathcal{O}] \le e^{\varepsilon} \Pr[\mathcal{C}(X'_t) \in \mathcal{O}] + \delta
\end{equation}
\end{definition}

\begin{theorem}\label{thm:dp-recent}
Assume that each $M_i$ in $\mathcal{E}_t$ is
produced by a model-agnostic $(\varepsilon_1,\delta)$-DP mechanism  $\mathcal{A}$ and that the noisy weight $w_i^*$ is produced by the  Laplace mechanism  $\mathcal{L}$ in Section \ref{sec:dp_ens_weight}.
Then the combined mechanism that produces 
$\mathcal{E}_t =\{M_{\tau(1)},\cdots,M_{\tau(k)}, w_{\tau(1)}^*,\cdots,w_{\tau(k)}^*\}$ is $(max\{\varepsilon_1, k \cdot \varepsilon_2\}, \delta)$-differentially private. 
\end{theorem}
\begin{proof}
For simplicity of proof, we write $\tau(1),\cdots,\tau(k)$ as $1,\cdots,k$.
The $(\varepsilon_1,\delta)$-DP guarantee 
of $\mathcal{A}$ implies that for neighboring training subsets $S_i,S'_i$ (Def. \ref{def:neighbor2}), and for any possible set  $\mathcal{M}_i$ of outputs: 
\begin{gather}
    \Pr[\mathcal{A}(S_i) \in \mathcal{M}_i] \leq e^{\varepsilon_1} \Pr[\mathcal{A}(S_i') \in \mathcal{M}_i]  + \delta \label{eq:tr_s1}
\end{gather}

%  \mathcal{M} \mathcal{W}

For the weight calculation, the Laplace mechanism $\mathcal{L}$ provides $(\varepsilon_2,0)$-DP for releasing the noisy weights $w^*_i$ following Theorem \ref{thm:lap}. therefore, for any $M_i \in \mathcal{M}_i$, neighboring validation subsets $V_t$ and $V_t'$, and any possible set $\mathcal{W}^*_i$ of weights:
\begin{gather}
    \Pr[\mathcal{L}(M_i(V_t)) \in \mathcal{W}^*_i] \leq e^{\varepsilon_2} \Pr[\mathcal{L}(M_i(V_t')) \in \mathcal{W}^*_i] \label{eq:lap_s1}
\end{gather}

% Let's define our overall input data $X = S_1 \cup ... \cup S_k \cup V_k$. Since the subsets are disjoint, % the DP definition allows changing only one of the datasets in $S_1 \cup ... \cup S_k \cup V_k$ to reach a neighboring overall dataset $X'$. Therefore, for a DP definition over a given possible output, which in this case is the collection of models and weights, 
% to reach a neighboring $X'$ we can either change one of $S_i$ used to build the models or change $V_k$ used to create the weights, but we cannot change more than one of those.

Denoting our combined mechanism as $\mathcal{C}$, with input $X_t=<S_1,\cdots,S_k,V_t>$, and any possible set of outputs $O = \{ \mathcal{M}_1,\cdots,\mathcal{M}_k, \\ \mathcal{W}^*_1,\cdots,\mathcal{W}^*_k \}$, we get:
\begin{gather}
    \nonumber \Pr[\mathcal{C}(X_t) \in O] = 
\\
    \nonumber \Pr[\mathcal{A}(S_1) \in \mathcal{M}_1] \cdot ... \cdot \Pr[\mathcal{A}(S_k) \in \mathcal{M}_k] \cdot
\\
    \Pr[\mathcal{L}(M_1(V_t)) \in \mathcal{W}^*_1] \cdot ... \cdot \Pr[\mathcal{L}(M_k(V_t)) \in \mathcal{W}^*_k] \label{eq:c_s1}
\end{gather}

Now consider the only two possible cases of neighboring  $X_t=<S_1,\cdots,S_k,V_t>$ and $X'_t=<S'_1,\cdots,S'_k,V'_t>$:
(I) change one arbitrary $S_i$ or (II) change $V_t$. 

For Case (I), all models should satisfy Equation~\eqref{eq:tr_s1}, but since only \textbf{one} $S_i$ changes to reach a neighboring input, Equation~\eqref{eq:tr_s1} will be obtained on one $M_i$ and for all the others $j \neq i$ we would get $\Pr[\mathcal{A}(S_j) \in \mathcal{M}_j] = \Pr[\mathcal{A}(S_j') \in \mathcal{M}_j]$ as $S_j = S_j'$. Additionally, since in this case we are \textbf{not} changing $V_t$, $V_t = V_t'$, so for all $1 \leq i \leq k$, $\Pr[\mathcal{L}(M_i(V_t)) \in \mathcal{W}^*_i] = \Pr[\mathcal{L}(M_i(V_t')) \in \mathcal{W}^*_i]$. Combining these two facts we reach $\Pr[\mathcal{C}(X_t) \in O] \leq e^{\varepsilon_1} \Pr[\mathcal{C}(X'_t) \in O] + \delta$ from Equation~\eqref{eq:c_s1} above.

For Case (II), we do not change any of $S_1,\cdots,S_k$, thus for all $1 \leq i \leq k$,  $\Pr[\mathcal{A}(S_i) \in \mathcal{M}_i] = \Pr[\mathcal{A}(S_i') \in \mathcal{M}_i]$. Additionally, changing $V_t$ for this scenario, every application of the Laplace mechanism satisfies Equation~\eqref{eq:lap_s1}, which is done $k$ times for $M_i$, $1 \leq i \leq k$. Combining these two facts we reach $\Pr[\mathcal{C}(X_t) \in O] \leq e^{k \varepsilon_2} \Pr[\mathcal{C}(X'_t) \in O]$ from Equation~\eqref{eq:c_s1} above.

Finally, since DP must hold for the worst-case guarantee, we take the maximum between the two cases defined above, which gives us the $(\max\{\varepsilon_1, k\cdot\varepsilon_2\}, \delta)$-DP. Note that combining the two cases is a tailored instantiation of the parallel composition (Theorem \ref{thm:pc}).
\end{proof}

\textbf{Discussion.} 
The construction of $(\varepsilon_1,\delta)$-DP mechanism  $\mathcal{A}$ for training a single model $M_i$ has been studied in the literature, for example, DP neural networks \cite{abadi2016deep}, DP random forest \cite{rana2015differentially}, and DP support-vector machine \cite{rubinstein2009learning}. Our focus is on the construction of $(\varepsilon,\delta)$-DP mechanisms $\mathcal{L}$ for building an ensemble $\mathcal{E}_t=\{M_{\tau(1)},\cdots,M_{\tau(k)},  w^*_{\tau(1)},\cdots, w^*_{\tau(k)}\}$, using the single model mechanism $\mathcal{A}$ as a black-box.

\subsection{Releasing the History of Ensembles}\label{sec:dp_t_time}

\begin{algorithm}[h]
\caption{Update The Ensemble}\label{algo:ens_update}
\begin{algorithmic}[1]
\REQUIRE The ensemble $\mathcal{E}_t=\{M_{\tau(1)},\cdots,M_{\tau(k)}, w^*_{\tau(1)},\cdots, w^*_{\tau(k)}\}$; the training and validation subsets at $t+1$, i.e., $S_{t+1}$ and $V_{t+1}$; 
privacy parameters ($\varepsilon_1, \varepsilon_2, \delta$); update\_mode (oldest or worst).

\ENSURE $\mathcal{E}_{t+1}$
\STATE Train $(\varepsilon_1,\delta)$-DP model $M_{t+1}$ on $S_{t+1}$
\IF {update\_mode = ``oldest"}
    % \STATE Calculate $w^*_{t-k+2},...,w^*_{t+1}$ using $M_{t-k+2},...,M_{t+1}$, $V_{t+1}$, and $\varepsilon_2$
    \STATE Calculate $w^*_{\tau(2)},\cdots,w^*_{\tau(k)},w^*_{t+1}$ for $M_{\tau(2)},\cdots,M_{\tau(k)},M_{t+1}$ using $V_{t+1}$ and $\varepsilon_2$
    \STATE $w^*_{\tau(1)},\cdots,w^*_{\tau(k-1)},w^*_{\tau(k)}$ $\leftarrow$ $w^*_{\tau(2)},\cdots,w^*_{\tau(k)},w^*_{t+1}$ 
        % \algorithmiccomment{Reindex}
        \STATE $M_{\tau(1)},\cdots,M_{\tau(k-1)},M_{\tau(k)}$ $\leftarrow$ $M_{\tau(2)},\cdots,M_{\tau(k)},M_{t+1}$
    % \algorithmiccomment{Reindex}
        \STATE $\mathcal{E}_{t+1}=\{M_{\tau(1)},\cdots,M_{\tau(k)}, w^*_{\tau(1)},\cdots, w^*_{\tau(k)}\}$
        
    % \algorithmiccomment{Most recent $k$ $M_i$s and associated $w^*_i$s}
    
    % \STATE $\mathcal{E}_{t+1} \leftarrow \{M_{t-k+2},...,M_{t+1}, w^*_{t-k+2},...,w^*_{t+1}\}$ 
    \ELSE
    % \STATE Calculate $w^*_{t-k+1},...,w^*_{t+1}$ using $M_{t-k+1},...,M_{t+1}$,  $V_{k+1}$, and $\varepsilon_2$
        \STATE Calculate $w^*_{\tau(1)},\cdots,w^*_{\tau(k)},w^*_{t+1}$ for  $M_{\tau(1)},\cdots,M_{\tau(k)},M_{t+1}$ using $V_{t+1}$ and $\varepsilon_2$
         \STATE $w^*_{\tau(1)},\cdots,w^*_{\tau(k)},w^*_{\tau(k+1)}$ $\leftarrow$ $w^*_{\tau(1)},\cdots,w^*_{\tau(k)},w^*_{t+1}$ 
        % \algorithmiccomment{Reindex}
        % \STATE $M_{\tau(1)},...,M_{\tau(k-1)},M_{\tau(k)}$ $\leftarrow$ $M_{\tau(2)},...,M_{\tau(k)},M_{t+1}$

     \STATE $i^* \leftarrow argmin_i\{w^*_{\tau(i)} \mid 1 \leq i\leq k+1\}$
     \STATE $\mathcal{E}_{t+1} \leftarrow \{M_{\tau(i)},w^*_{\tau(i)} \mid 1 \leq i\leq k+1, i\neq i^*\}$ 
    %  \algorithmiccomment{Top weighted $k$ $M_i$s and associated $w^*_i$s}
\ENDIF 
\STATE return $\mathcal{E}_{t+1}$
\end{algorithmic}
\end{algorithm}

Algorithm \ref{algo:ens_update} shows the steps for updating the ensemble $\mathcal{E}_t$ 
to adapt the new chunk $D_{t+1}$ for two update modes, indicated by the input variable update\_mode: replace the oldest model and replace the worst model (i.e., the model having smallest $w^*_i$). In the former case $V_{t+1}$ is used $k$ times (Step 3), and in the latter case $V_{t+1}$ is used $k+1$ times (Step 8). Theorem \ref{thm:dp-recent} shows that releasing a single ensemble $\mathcal{E}_t$ satisfies $(\max\{\varepsilon_1,k \cdot \varepsilon_2\}, \delta)$-DP. With the repeated update at each time $t$, the adversary is able to access the history of all released ensembles up to the current time. We show that, with the access to the history, the $(\max\{\varepsilon_1,k \cdot \varepsilon_2\}, \delta)$-DP remains to hold in the case of replacing 
oldest model, and degrades to $(\max\{\varepsilon_1,(k+1) \cdot \varepsilon_2\}, \delta)$-DP in the case of replacing worst model.

First, we extend the notion of DP to the global input data $X$ from time $1$ to time $t$, i.e., $X=<X_1,\cdots,X_t>$ where $X_i$ is the input data for the ensemble $\mathcal{E}_i$ defined in Definition \ref{def:neighbor2}. We say that $X$ and $X'$ are \textit{neighboring} if exactly one pair $(X_i,X'_i)$ is neighboring, as defined in Definition \ref{def:neighbor2}, and for all other $j\neq i$, $X_j=X'_j$. We consider the output consisting of all ensembles released up to the time $t$, i.e., $\mathcal{E}=<\mathcal{E}_1,\cdots,\mathcal{E}_t>$. 

\begin{definition}[Differential Privacy for History]\label{def:DP-ensemble-history}
A mechanism $\mathcal{C}$ from the domain of $X$ to the range of $\mathcal{E}$ is $(\varepsilon,\delta)$-differentially private with respect to history if for any neighbouring pair $(X, X')$ and for all sets $\mathcal{O}$ of possible outputs:
\begin{equation}
    \Pr[\mathcal{C}(X) \in \mathcal{O}] \le e^{\varepsilon} \Pr[\mathcal{C}(X') \in \mathcal{O}] + \delta
\end{equation}
\end{definition}

\begin{theorem}\label{thm:dp-update-old}
With update\_mode=``oldest", Algorithm \ref{algo:ens_update} is $(\max\{\varepsilon_1,k \cdot \varepsilon_2\}, \delta)$-DP with respect to history.
\end{theorem}
\begin{proof}
The proof is basically the same as for Theorem \ref{thm:dp-recent}, noting that $X$ and $X'$ differ only in a single sample either in the training subset or in the validation subset, for one ensemble.
\end{proof}

Therefore, even if the adversary has access to all released ensembles, the privacy loss does not accumulate compared to releasing a single ensemble. This is due to the two facts. (i)  each model in an ensemble is trained on a \textit{disjoint} training subset, which ensures that accessing more models does not change the $(\varepsilon_1,\delta)$-DP (i.e., parallel composition, Theorem \ref{thm:pc}), (ii) each validation subset is used exactly \textit{$k$ times} (that is, $V_t$ is used for the $k$ models in $\mathcal{E}_t$), which ensures the $(k\cdot\varepsilon_2,0)$-DP remains unchanged.

\begin{theorem}\label{thm:dp-update-worse}
With update\_mode=``worst", Algorithm \ref{algo:ens_update} is $(\max\{\varepsilon_1, \\ (k+1) \cdot \varepsilon_2\}, \delta)$-DP with respect to history.
\end{theorem}
\begin{proof}
The proof follows the same idea as Theorem \ref{thm:dp-update-old}, but for the case of replacing the worst model, we have to calculate the weights for all $k$ models already in the ensemble \emph{plus} the additional new model in order to find the worst model, so each validation subset is used $k+1$ times. Therefore, now we have the overall privacy guarantee of $( \max\{\varepsilon_1,(k+1) \cdot \varepsilon_2\}, \delta)$-DP.
\end{proof}

\textbf{Discussion.} Therefore, replacing worst model incurs a \textit{slightly} larger privacy loss, compared to  replacing oldest model. Importantly, in both cases the privacy loss depends on the size of an ensemble, $k$, but not on the number of ensembles released. This property is essential for practical use because the number of ensemble updates is potentially unbounded for data streams. To optimize the given privacy budget $(\varepsilon,\delta)$, 
we can set $\varepsilon_1 = \varepsilon$ and $\varepsilon_2 = \nicefrac{\varepsilon}{k}$ when replacing oldest model 
(Theorem \ref{thm:dp-update-old}), and set $\varepsilon_1 = \varepsilon$ and $\varepsilon_2 = \nicefrac{\varepsilon}{k+1}$ when replacing worst model (Theorem \ref{thm:dp-update-worse}).

\begin{table}[h]
\begin{tabular}{llllll}
\hline
Dataset        & Attr. & Obs. & C/R & Prop. & Type  \\ \hline \hline
Hyperplane     & 20         & Variable           & C & 50\% & Synthetic       \\
EMBER-B          & 2381       & 2,100,000       & C & 50\% & Real       \\
EMBER-U          & 2381       & 1,365,000       & C & 30\% & Real      \\
Housing Market & 292        & 30,473        & R  & NA & Real        \\ \hline
\end{tabular}
\caption{C/R for classification/regression and Prop. for the proportion of positive class.}\label{tab:datasets}
\vspace{-0.8cm}
\end{table}

\section{Evaluation}\label{sec:DP-transfer}
This section evaluates the proposed DP temporal ensemble method. We train each $M_i$ as a neural network
using DPSGD \cite{abadi2016deep} with privacy budget $(\varepsilon_1,\delta)$. In each iteration, DPSGD 
adds the Gaussian noise $\mathcal{N}(0,\sigma^2C^2)$ to the clipped gradient $\frac{{g(x_i)}}{{\text{max}(1,\nicefrac{||g(x_i)||_2}{C})}}$
% (Line 6 in \cite{abadi2016deep}, Algorithm 1), 
where $C$ is the clipping factor. 
For a large number of model parameters, $||g(x_i)||_2$ is large, leading to a noisy gradient. This effect is compounded for typically small data chunk sizes
where the sampling ratio for a fixed minibatch size 
% ($q$ in Theorem 1 of \cite{abadi2016deep}) 
becomes relatively large, which increases $\sigma$. To reduce the norm $||g(x_i)||_2$, we also consider the option of  \textit{transfer learning} for training $M_i$: first, we pre-train a model using a public dataset $P$ without privacy concerns (for example, obsolete dataset, anonymized dataset, dataset obtained with data owners' explicit consent, or dataset from related but public domain) and then, we train only the last few layers using the sensitive $D_i$ via DPSGD keeping the parameters for other layers unchanged. If no such public $P$ is available, $M_i$ will be fully trained using the sensitive $D_i$ via DPSGD.

\subsection{Data and Model Details}\label{sec:datasets}
Table \ref{tab:datasets} shows the data summary. 

\subsubsection{Hyperplane} Hyperplane is a \textit{synthetic} dataset used extensively in the concept drift literature \cite{hulten2001mining,wang2003mining} to classify points separated by a hyperplane. We simulate time-evolving concepts by changing the orientation and the position of the hyperplane in a smooth manner. 
As in \cite{gozuaccik2019unsupervised}, we use the \verb|HyperplaneGenerator()| 
function from \cite{10.5555/3291125.3309634}
% \footnote{See \url{https://bit.ly/2VEBsmK} for more details.}
to create the simulated points, and use four parameters (\verb|n_drift_features|, \verb|mag_change| , \verb|noise_percentage|, and \verb|sigma_percentage|) to  generate four drift types: \emph{gradual} drift (concept changes slowly over time)\footnote{Parameter values: <$10, 0.1, 0.05, 0.1$>}, \emph{rapid} drift (change happens at a rapid pace)\footnote{Parameter values: <$20, 0.4, 0.1, 0.4$>}, \emph{recurrent} drift (concepts reappear at future times, every fifth time for our case)\footnote{Same parameters as in rapid drift, use restart() argument every fifth time}, and \emph{abrupt} drift (concept changes suddenly at a time instance, every fifth time for our case)\footnote{Parameter values: <$0,0,0,0$>, switch labels every fifth time}. We evaluate using the classification Accuracy for the general setting. For all drift types, we generate a total of 20 chunks $D_i$s with the default size of 1000, and use a fully connected neural network with two hidden layers of sizes 20 and 10 respectively with ReLU as the activation function for the hidden layers and softmax for the output layer. The default drift type is rapid. We do not use any public data or transfer learning for this dataset.

% which has four parameters: \verb|n_drift_features|, \verb|mag_change| , \verb|noise_percentage|, and \verb|sigma_percentage|
% \footnote{See \url{https://bit.ly/2VEBsmK} for parameter details.}.
% For \emph{gradual} drift,
% % (concept changes slowly over time), 
% we set these parameters to <$10, 0.1, 0.05, 0.1$>; for \emph{rapid} drift,
% % (change happens at a rapid pace), 
% we set to <$20, 0.4, 0.1, 0.4$>; for \emph{recurrent} drift (concepts reappear at future times), we use the same settings as rapid drift but we restart the drift every fifth time using the \verb|restart| argument; and for \emph{abrupt} drift (concept changes suddenly at a time instance), we use <$0,0,0,0$> and switch the labels every tenth time. 
% For all drift types, we generate a total of 20 chunks $D_i$s with the default size of 1000, and use a fully connected neural network with two hidden layers of sizes 20 and 10 respectively with ReLU as the activation function for the hidden layers and softmax for the output layer. The default drift type is rapid. We do not use any public data or transfer learning for this dataset.

\subsubsection{EMBER-(B \& U)}\label{sec:res_ember} EMBER  \cite{2018arXiv180404637A} contains features for
Windows executable files for the years of 2017 and 2018 with the goal to classify malicious vs benign files, and the dataset has a natural concept drift \cite{yang2021bodmas}. We remove unlabelled observations. \textbf{EMBER-B} is the original class-balanced version and \textbf{EMBER-U} is obtained by under-sampling the positive class to  30\%.
For EMBER-B, we evaluate using classification Accuracy,  and for EMBER-U we evaluate using Balanced Accuracy (BA) with $a_1=0.7$ and $a_2=0.3$ (Eqn. \eqref{eqn:BA}). The data chunks $D_i$s are created as bi-weekly observations, leading to an average chunk size of 33,333 for EMBER-B and 21,666 for EMBER-U. A fully connected neural network is selected via hyperparameter search \footnote{Four hidden layers (1400,2000,1100,250,2), ReLU for hidden, softmax for output}.
For transfer learning, we use the first six months of 2017 as the public data $P$, leave the last six months of 2017 as the \emph{time buffer}, and retrain the last two layers of the pre-trained model (preserving the layer sizes) using $D_i$s for 2018.

\subsubsection{Housing Market}\label{sec:res_house} Housing market \cite{sberbankdata}  
contains the property information from August 2011 to June 2015, with the goal of predicting the continuous property price (i.e., regression). We evaluate using 1-MSE where MSE is defined by Eqn. \eqref{eqn:mse_reg}.
We normalize the property price to within [0,1]. A fully connected neural network is selected via hyperparameter search\footnote{Five hidden layers (500,350,250,150,50,1), ReLU for hidden, sigmoid for output}. For transfer learning, we use the data from 2011 as public data $P$ to pre-train a model, leave out the data from 2012 as the time buffer, and use the months starting from January 2013 as our monthly data chunks $D_i$s, leading to an average chunk size of 859. $M_i$ is obtained by retraining the last two layers of the pre-trained model (preserving the sizes) using $D_i$. 

For \textit{all} datasets: we standardize continuous features using StandardScaler from scikit-learn\cite{scikit-learn} and use the one-hot encoding for categorical features. We run DPSGD with 30 epochs with the minibatch size of 100. The labeled data $D_i$ is split into train-validation-test using 70\%-20\%-10\% and we use the training subset for training the model, validation subset for weight estimation, and the test subset to report the performance. We report the average result of 10 runs with standard errors. 

%%%%%%%%%%%%%%%%%%%%%%%%%%%%%%%%%%%%%%%%%%%%%%%%%%%
% moving figures here for better placement in pdf %
%%%%%%%%%%%%%%%%%%%%%%%%%%%%%%%%%%%%%%%%%%%%%%%%%%%
\begin{figure*}[t]
 \centering
         \begin{subfigure}{.25\textwidth}
     \includegraphics[scale=0.3]{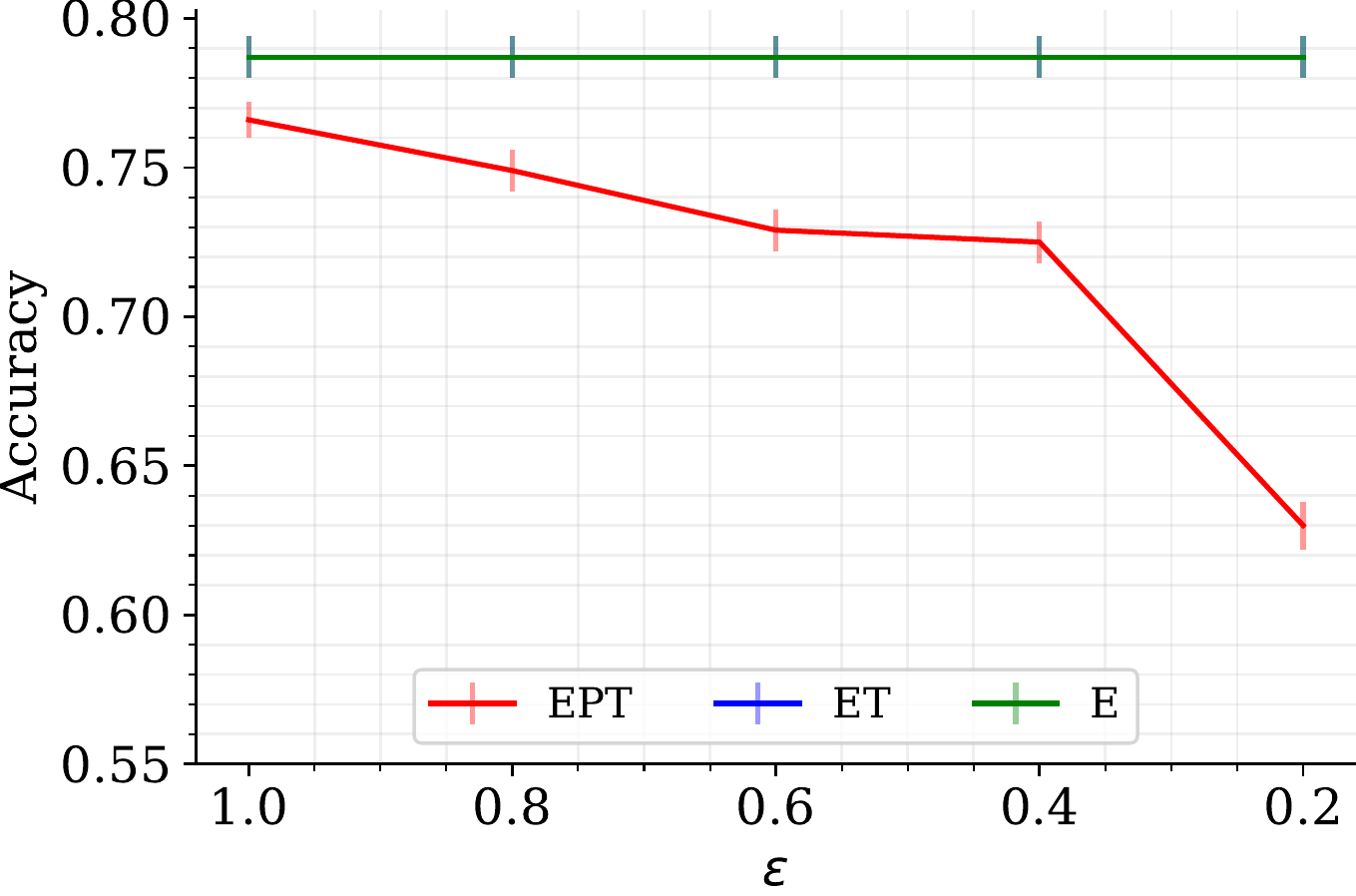}
        \caption{Hyperplane }
    \end{subfigure}%
    \begin{subfigure}{.25\textwidth}
     \includegraphics[scale=0.3]{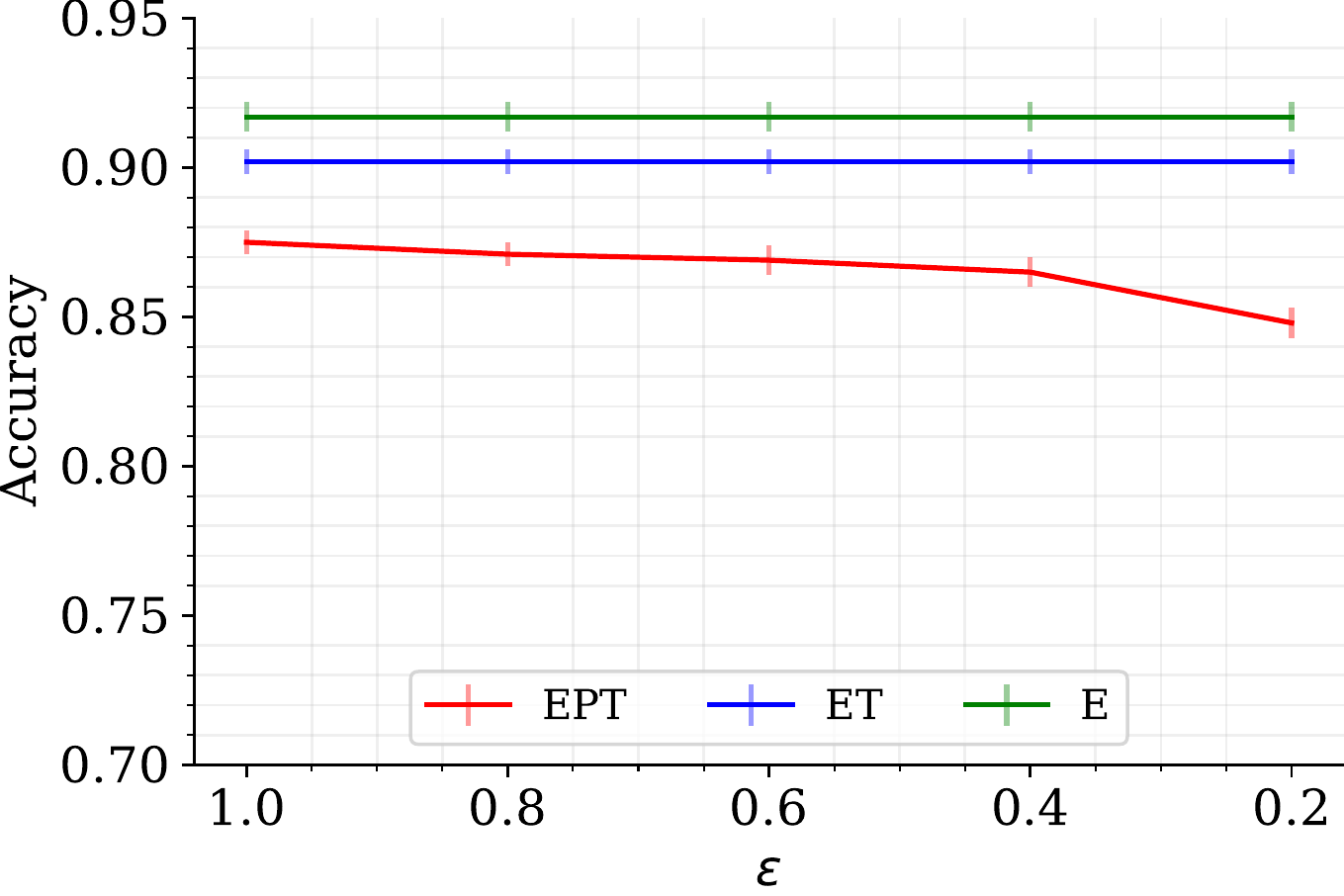}
        \caption{EMBER-B}
    \end{subfigure}%
        \begin{subfigure}{.25\textwidth}
     \includegraphics[scale=0.3]{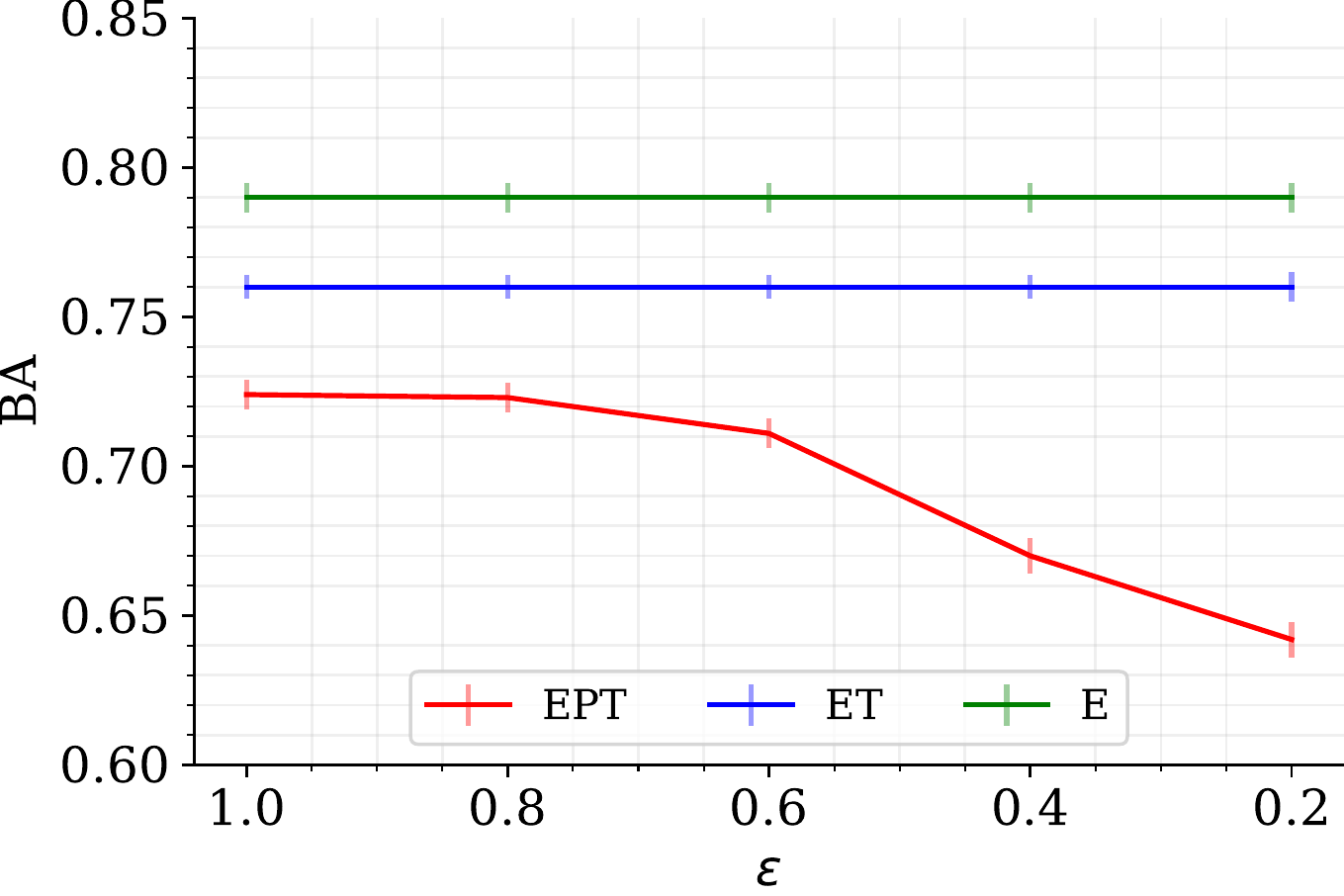}
        \caption{EMBER-U}
    \end{subfigure}%
        \begin{subfigure}{.25\textwidth}
     \includegraphics[scale=0.3]{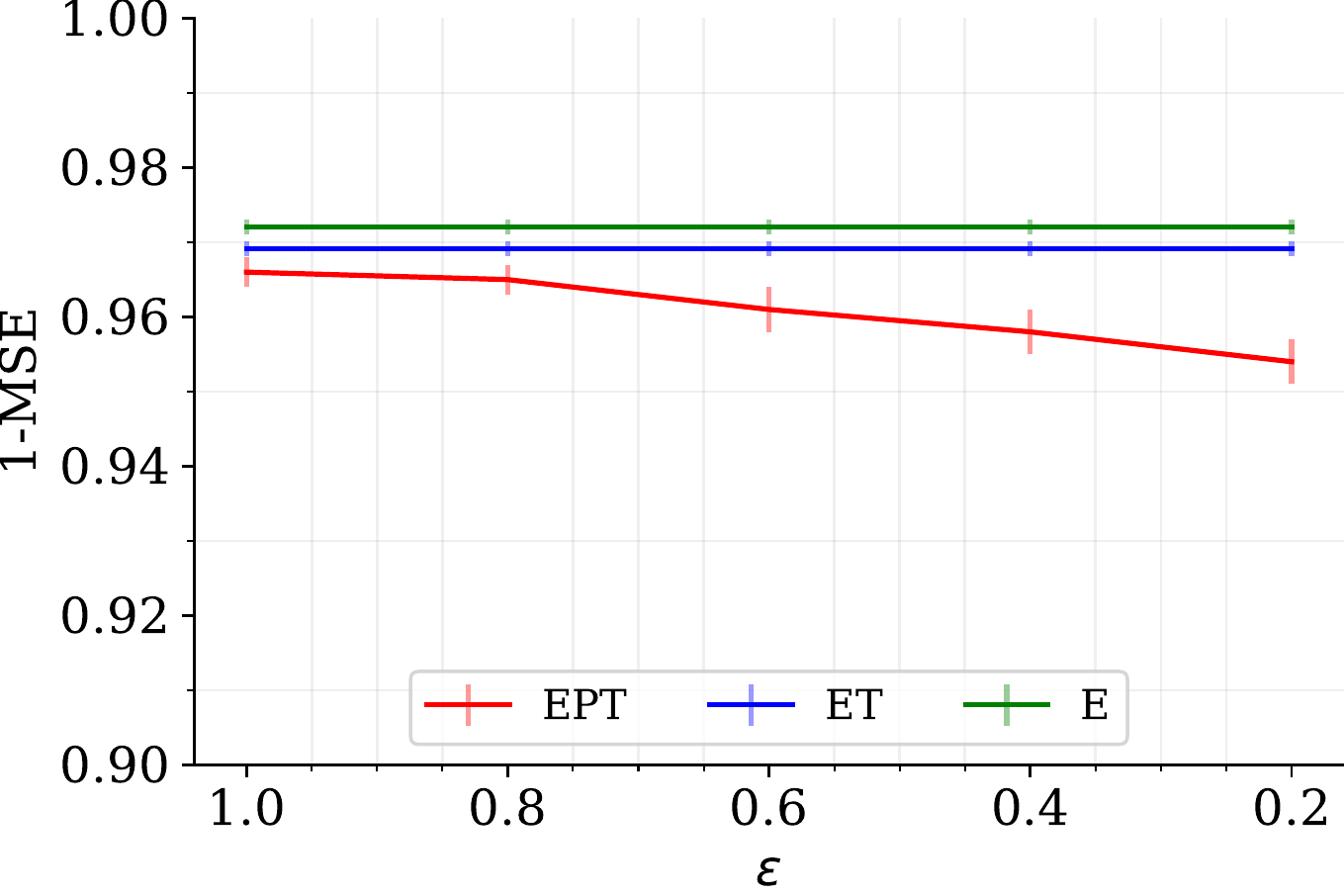}
        \caption{Housing Market }
    \end{subfigure}% 
    \vspace{-0.2cm}
    \caption{\textit{\textbf{Impact of privacy}} (Comparing EPT with ET and E). The vertical bars represent the standard errors. For Hyperplane, there is no transfer learning, hence, EPT is same as EP and ET is same as E.
    }\label{fig:main_res}
\end{figure*}

\begin{figure*}[]
 \centering
     \begin{subfigure}{.25\textwidth}
     \includegraphics[scale=0.31]{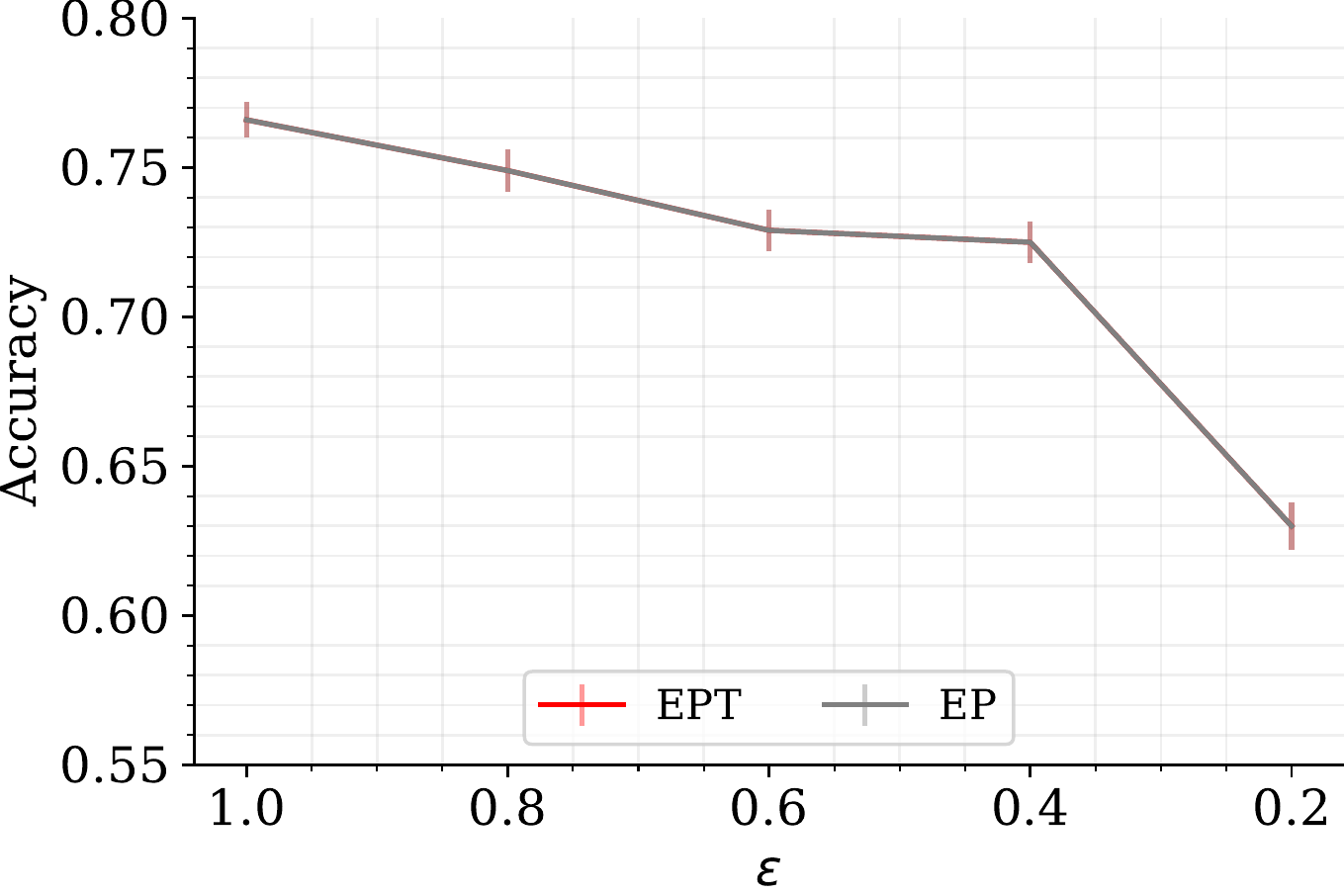}
        \caption{Hyperplane }
    \end{subfigure}%
    \begin{subfigure}{.25\textwidth}
     \includegraphics[scale=0.31]{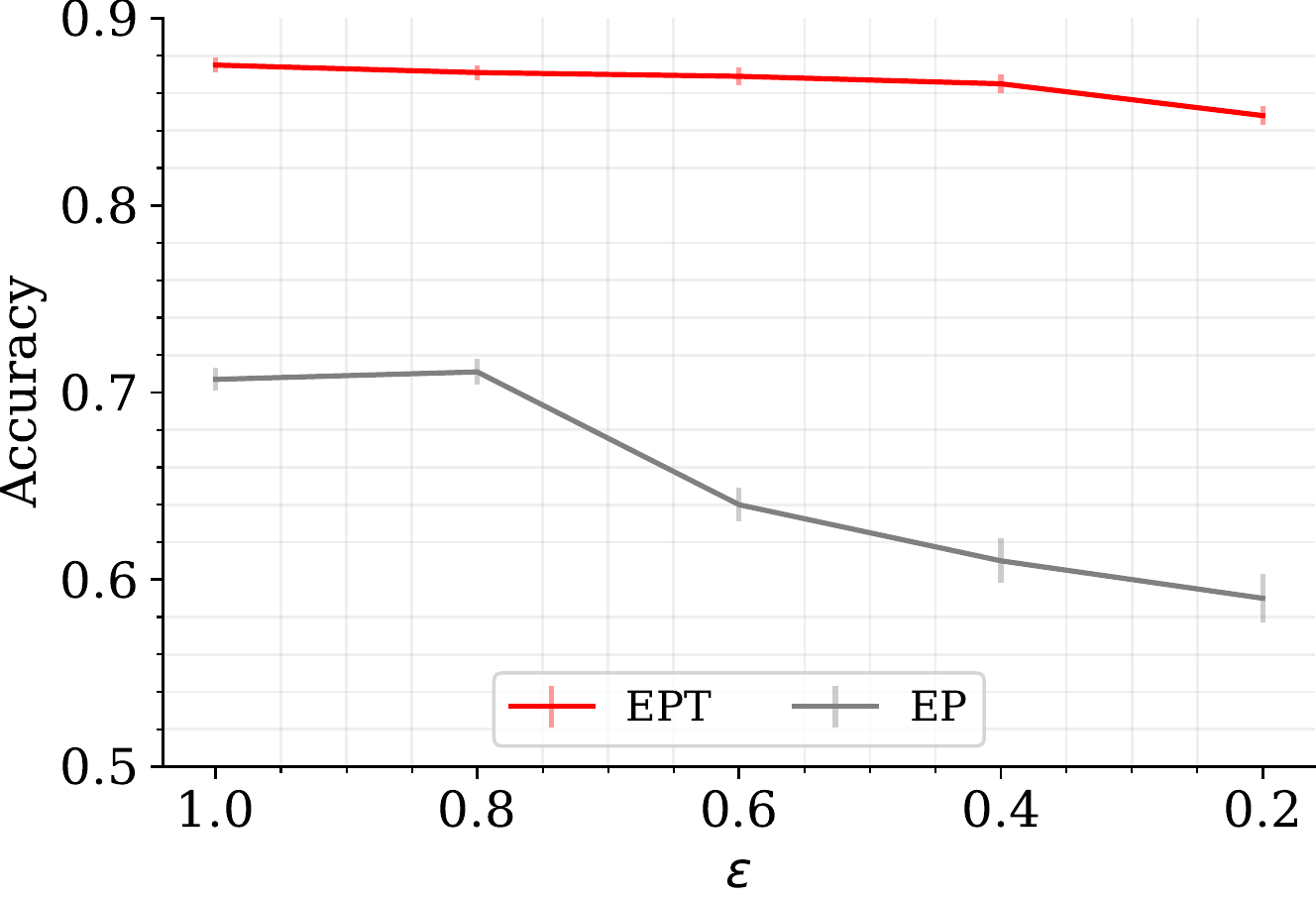}
        \caption{EMBER-B }
    \end{subfigure}%
         \begin{subfigure}{.25\textwidth}
      \includegraphics[scale=0.31]{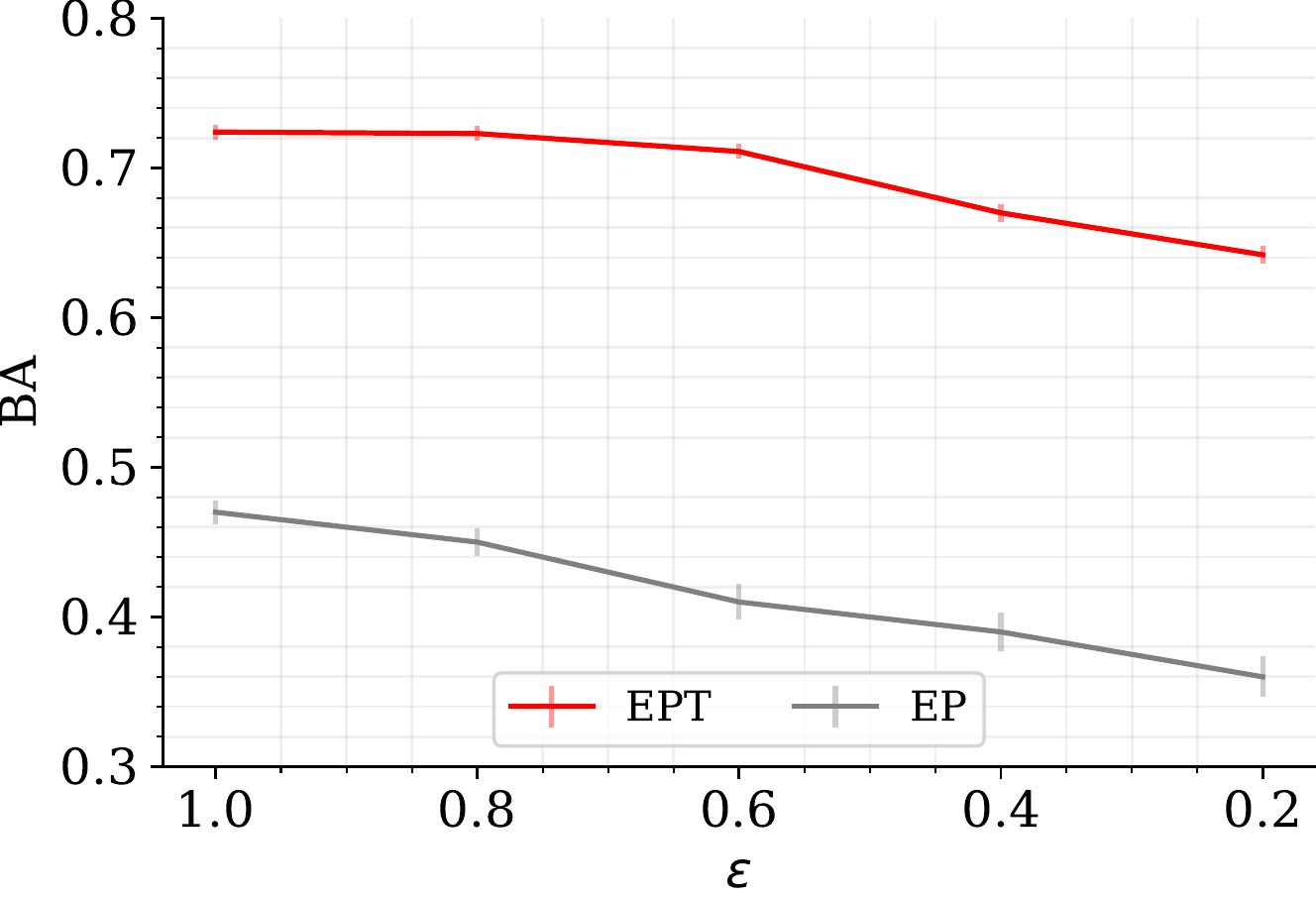}
         \caption{EMBER-U }
     \end{subfigure}%
        \begin{subfigure}{.25\textwidth}
     \includegraphics[scale=0.31]{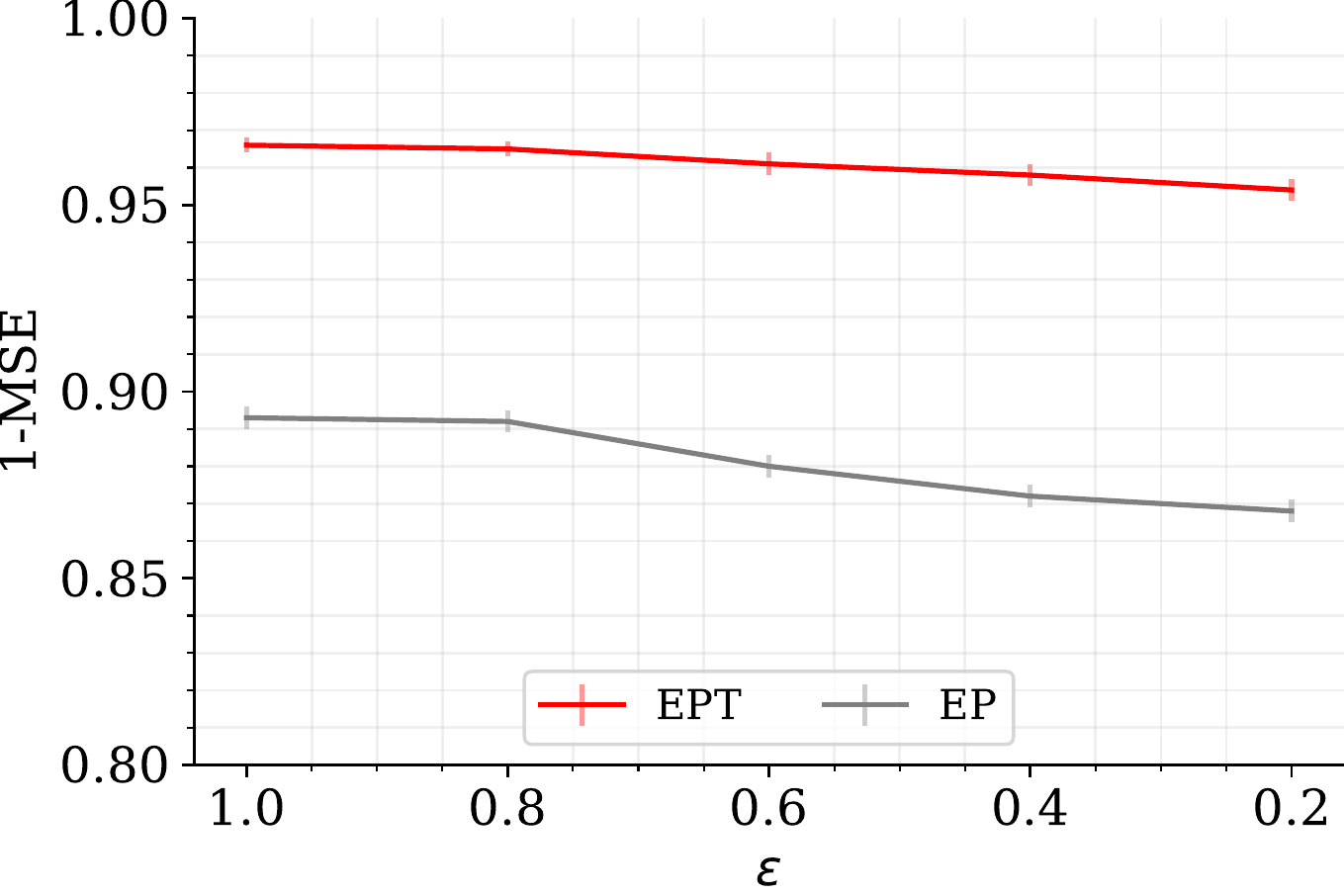}
        \caption{Housing Market }
    \end{subfigure}% 
\vspace{-0.2cm}
    \caption{\textit{\textbf{Impact of transfer learning}} (Comparing EPT vs EP). EPT and EP are identical for Hyperplane.
    }\label{fig:ab_trans}
    \vspace{-0.2cm}
\end{figure*}

\begin{table}[h]
\begin{tabular}{lllll}
\hline
  Method      & Private & Ensemble & Transfer & Data\\ \hline \hline
EPT     & \cmark        & \cmark         & \cmark         & [$D_{\tau(1)},\cdots,D_{\tau(k)}$]      \\
EP      & \cmark        &  \cmark        & \xmark          & [$D_{\tau(1)},\cdots,D_{\tau(k)}$]      \\
PT($1$) & \cmark        & \xmark          & \cmark          & [$D_{\tau(1)}$]      \\
PT($k$) & \cmark        & \xmark         & \cmark          & [$D_{\tau(1)} \cup \cdots \cup D_{\tau(k)}$]     \\
ET      & \xmark        & \cmark         & \cmark          & [$D_{\tau(1)},\cdots,D_{\tau(k)}$]     \\
E       & \xmark        & \cmark         & \xmark          &  [$D_{\tau(1)},\cdots,D_{\tau(k)}$]    \\ \hline
\end{tabular}\caption{Competitor methods. \cmark ~ signifies the presence of a trait whereas \xmark ~  signifies its absence. Data is the data for training an ensemble or training a model for non-ensembles.
}\label{tab:methods}
\vspace{-0.7cm}
\end{table}

\subsection{Competitor Methods}\label{sec:baselines}
As discussed in Section \ref{sec:related}, existing works on data streams either deal with summary statistics or do not consider privacy, or cannot deal with an unbounded number of updates. Table \ref{tab:methods} lists the methods evaluated and their characteristics. We use the following naming convention:  \textbf{``E"} denotes ensemble classifiers, \textbf{``P"} 
denotes DP, and \textbf{``T"} denotes transfer learning. 

\textbf{EPT} is the DP temporal ensemble proposed in Section \ref{sec:dp_ens} consisting of $k$ models ($M_i$s) trained on the $k$ data chunks $D_{\tau(1)},\cdots,D_{\tau(k)}$ with  transfer learning, whereas \textbf{EP} does not use transfer learning. 
\textbf{ET} and \textbf{E} are the non-private versions of EPT and EP and they serve as an upper bound for the performance of EPT and EP. We also compare our methods with two \emph{non-ensemble} solutions, \textbf{PT($1$)} and \textbf{PT($k$)}, where a \textit{single} model is used for prediction. PT($k$) uses the union $D_{\tau(1)} \cup \cdots \cup D_{\tau(k)}$ of $k$ chunks  
to train the model and advances to the \textit{next non-overlapping} window covering times $t+1,\cdots,t+k$, and PT($1$) is the special case of $k=1$, i.e., building a new model using each new chunk. With a single model trained using non-overlapping chunks, 
these methods do not need weight estimation and will spend the whole privacy budget on training the model. For prediction, EPT, EP, ET, E, and PT($1$) are used to predict in the \textit{next time} (i.e., $t+1$) whereas PT($k$) predicts in its next window (i.e., $t+1,\cdots,t+k$).

All DP methods are evaluated under the same privacy budget ($\varepsilon,\delta$). The following default settings  are used: ($\varepsilon=1$, $\delta=0.0001$),  window size $k=5$, drift type = ``rapid", and update\_mode = ``oldest". 
% Following Theorem \ref{thm:dp-update-old} and Theorem \ref{thm:dp-update-worse}, 
We set $\varepsilon_1 = \varepsilon$ and $\varepsilon_2 = \nicefrac{\varepsilon}{k}$ for update\_mode = ``oldest", and set $\varepsilon_1 = \varepsilon$ and $\varepsilon_2 = \nicefrac{\varepsilon}{k+1}$ for update\_mode = ``worst".  
We begin training for \emph{all} methods once we have the \emph{first} $k$ data chunks. This delay is only for evaluation purposes. 

Section \ref{sec:first_comp} studies the utility loss of our DP method compared to  \emph{non-private} counterparts, followed by the ablation studies evaluating the impact of transfer learning (Section \ref{sec:transfer}), ensemble (Section \ref{sec:ensemble_k}), drift dynamics and chunk sizes (Section \ref{sec:ablation_drift}), and update mode (Section \ref{sec:newvsold}).

\subsection{Impact of Privacy Preservation}\label{sec:first_comp}
The first question is how privacy preservation impacts the performance. To answer this question, we compare the performance of EPT with the non-private counterparts ET and E in Figure \ref{fig:main_res}. The main finding is that  EPT provides \emph{close} utility (average difference of $< 3\%$) to ET for $\varepsilon=1$. 
The utility gap increases as $\varepsilon$ decreases, with the average drop of 9\%  at $\varepsilon =0.2$. Our privacy settings are much tighter than those in the DPNN literature, for example, the minimum and maximum values of $\varepsilon$ are 2 and 100 according to the survey \cite{jayaraman2019evaluating}.
% , for example, the average $\varepsilon$ is 15 in  
Comparing the non-private models E and ET, transfer learning does not help. However, as we will show later, transfer learning significantly boosts the utility in the case of private models. 

%%%%%%%%%%%%%%%%%%%%%%%%%%%%%%%%%%%%%%%%%%%%%%%%%%%
% moving figures here for proper placement in pdf %
%%%%%%%%%%%%%%%%%%%%%%%%%%%%%%%%%%%%%%%%%%%%%%%%%%%
\begin{figure*}[]
 \centering
         \begin{subfigure}{.25\textwidth}
     \includegraphics[scale=0.31]{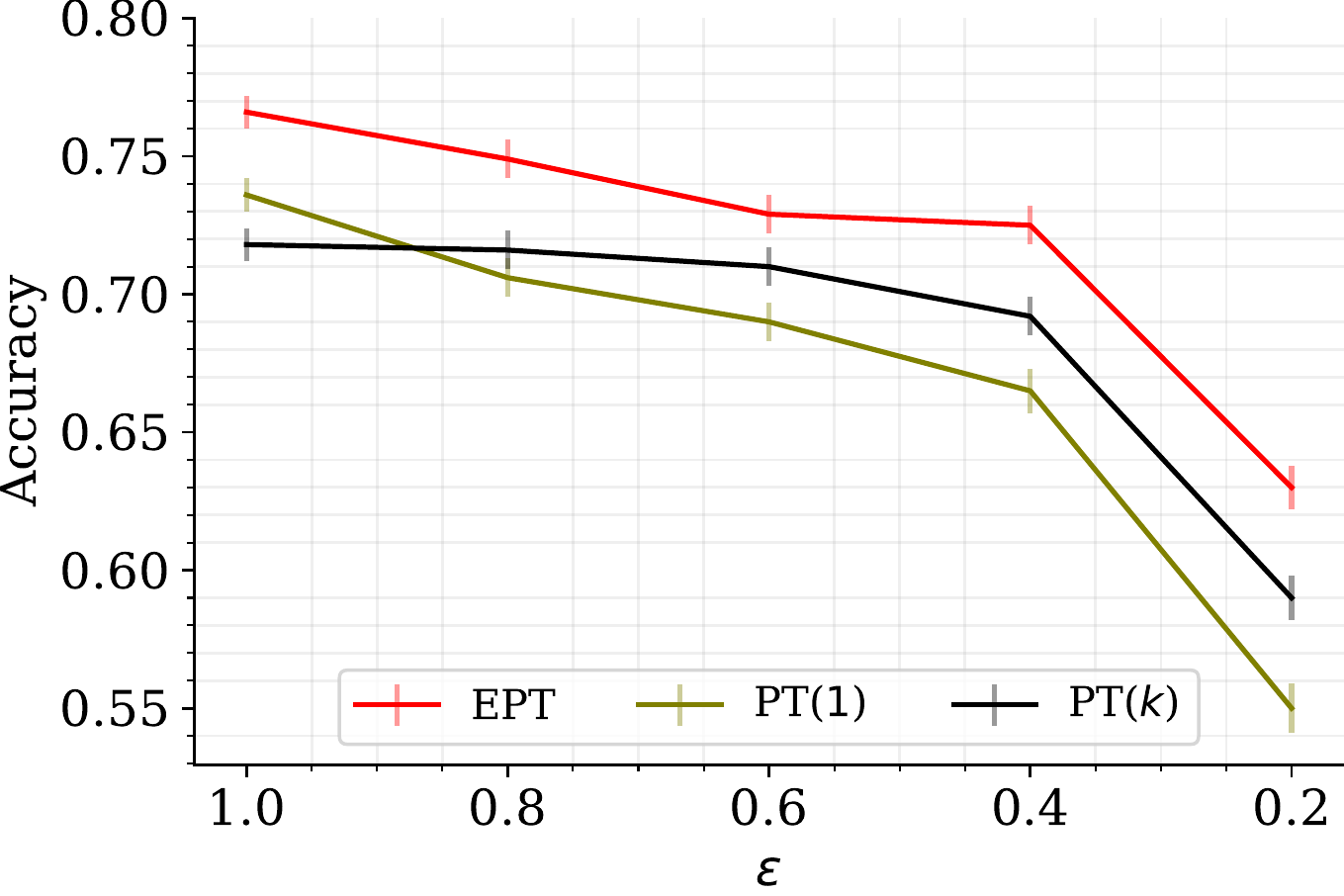}
        \caption{Hyperplane}
    \end{subfigure}% 
    \begin{subfigure}{.25\textwidth}
     \includegraphics[scale=0.31]{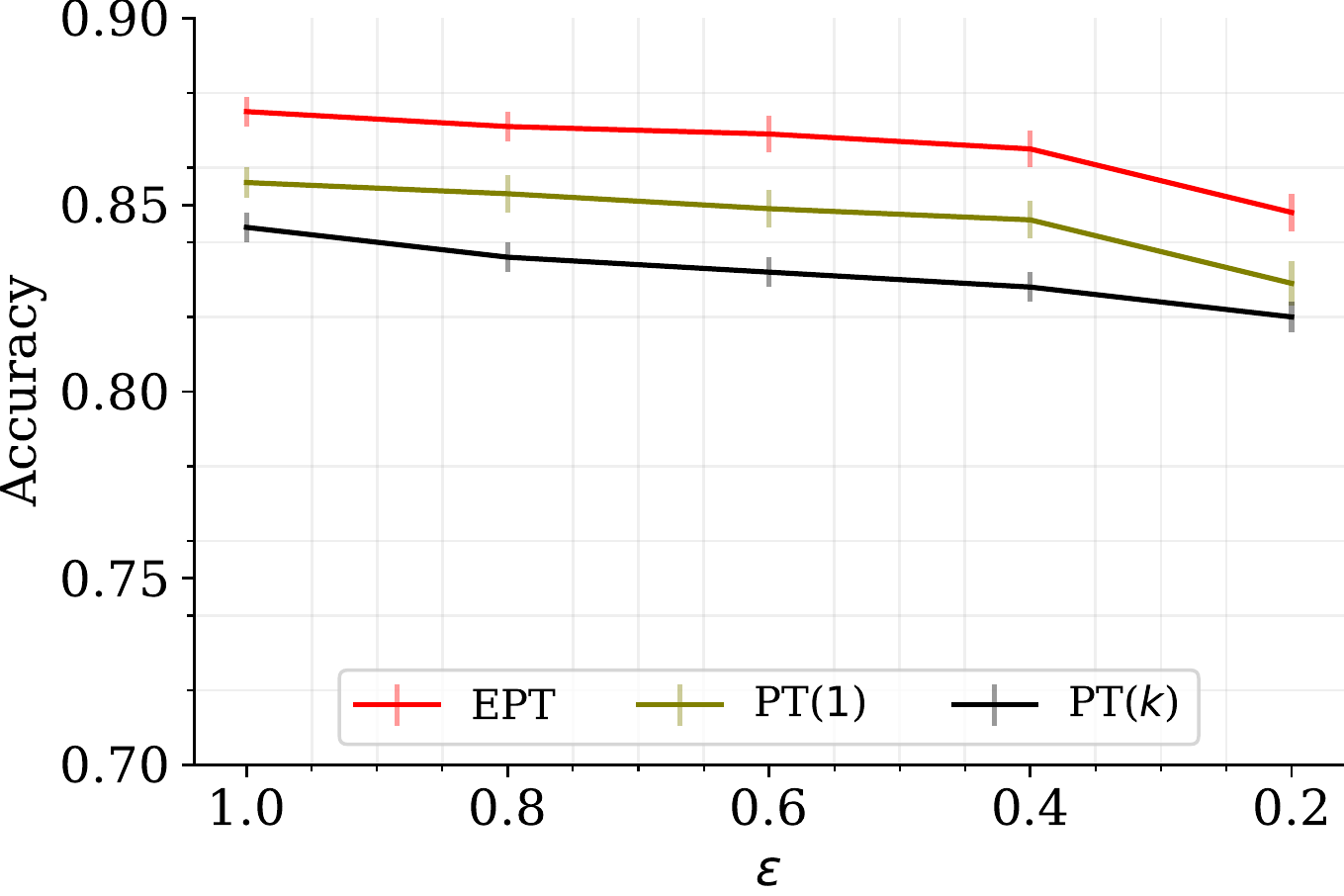}
        \caption{EMBER-B}
    \end{subfigure}%
      \begin{subfigure}{.25\textwidth}
     \includegraphics[scale=0.31]{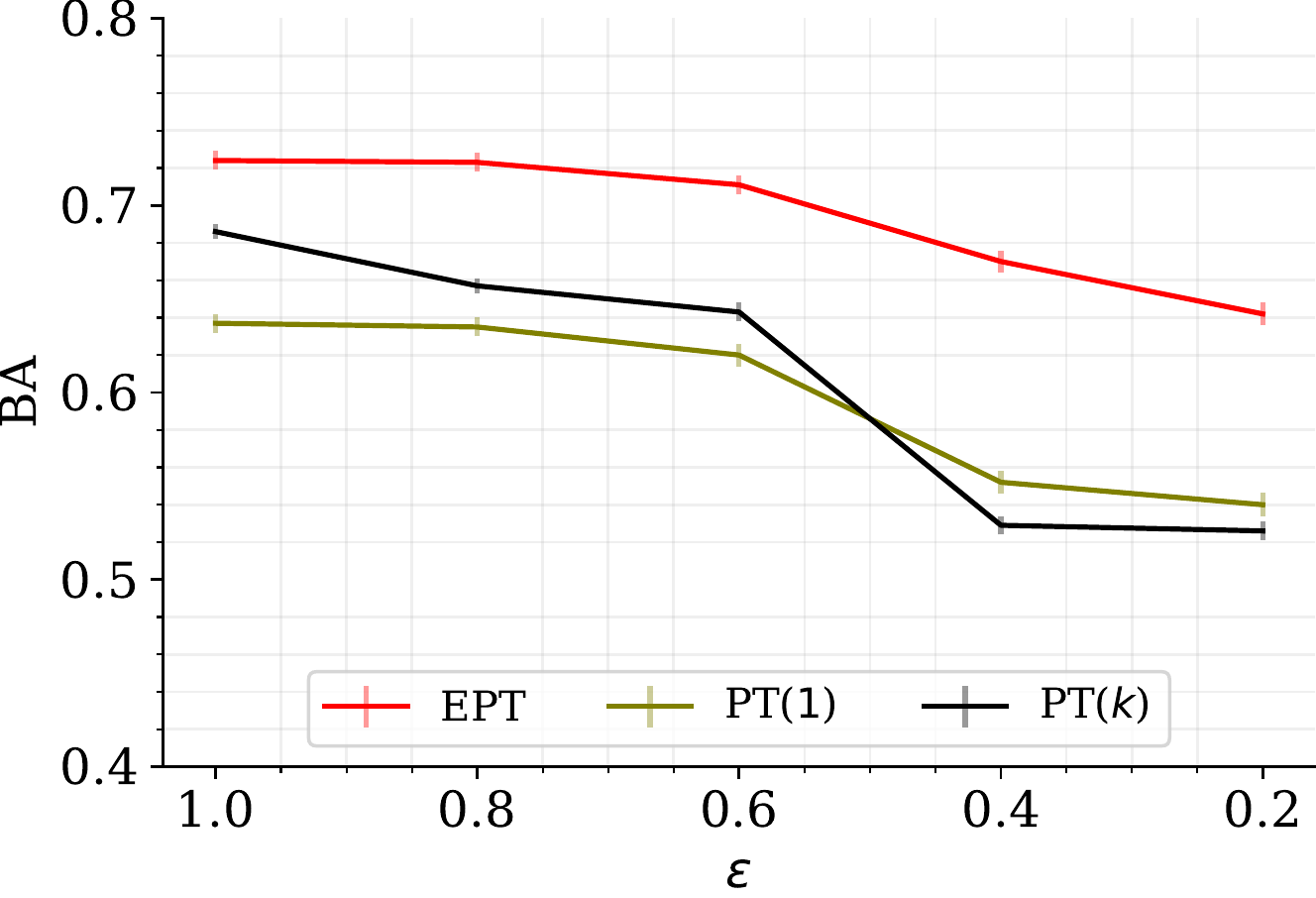}
        \caption{EMBER-U}
    \end{subfigure}%
    \begin{subfigure}{.25\textwidth}
     \includegraphics[scale=0.31]{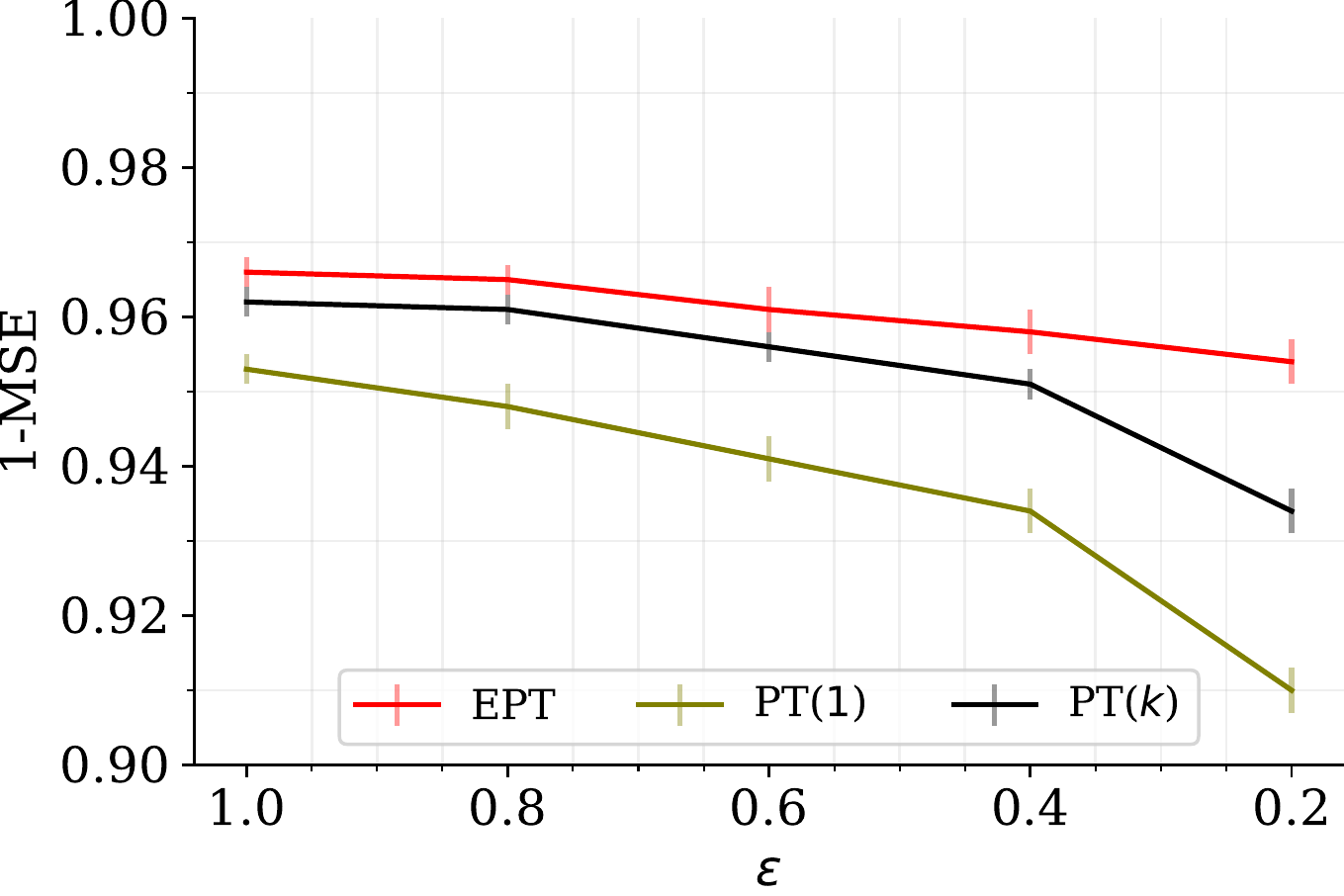}
        \caption{Housing Market}
        \end{subfigure}
    \medskip\\
            \begin{subfigure}{.25\textwidth}
     \includegraphics[scale=0.31]{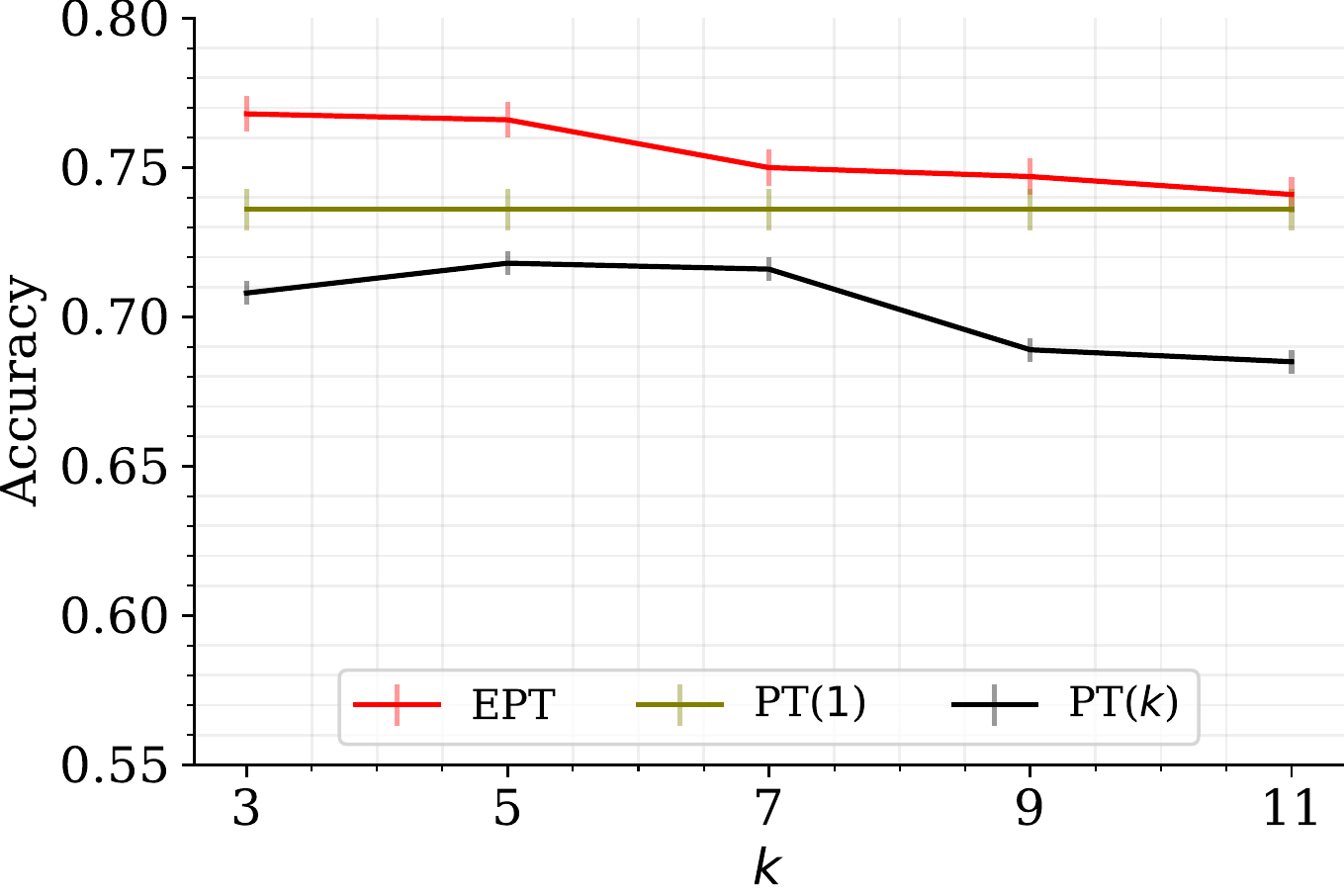}
        \caption{Hyperplane}
    \end{subfigure}% 
        \begin{subfigure}{.25\textwidth}
     \includegraphics[scale=0.31]{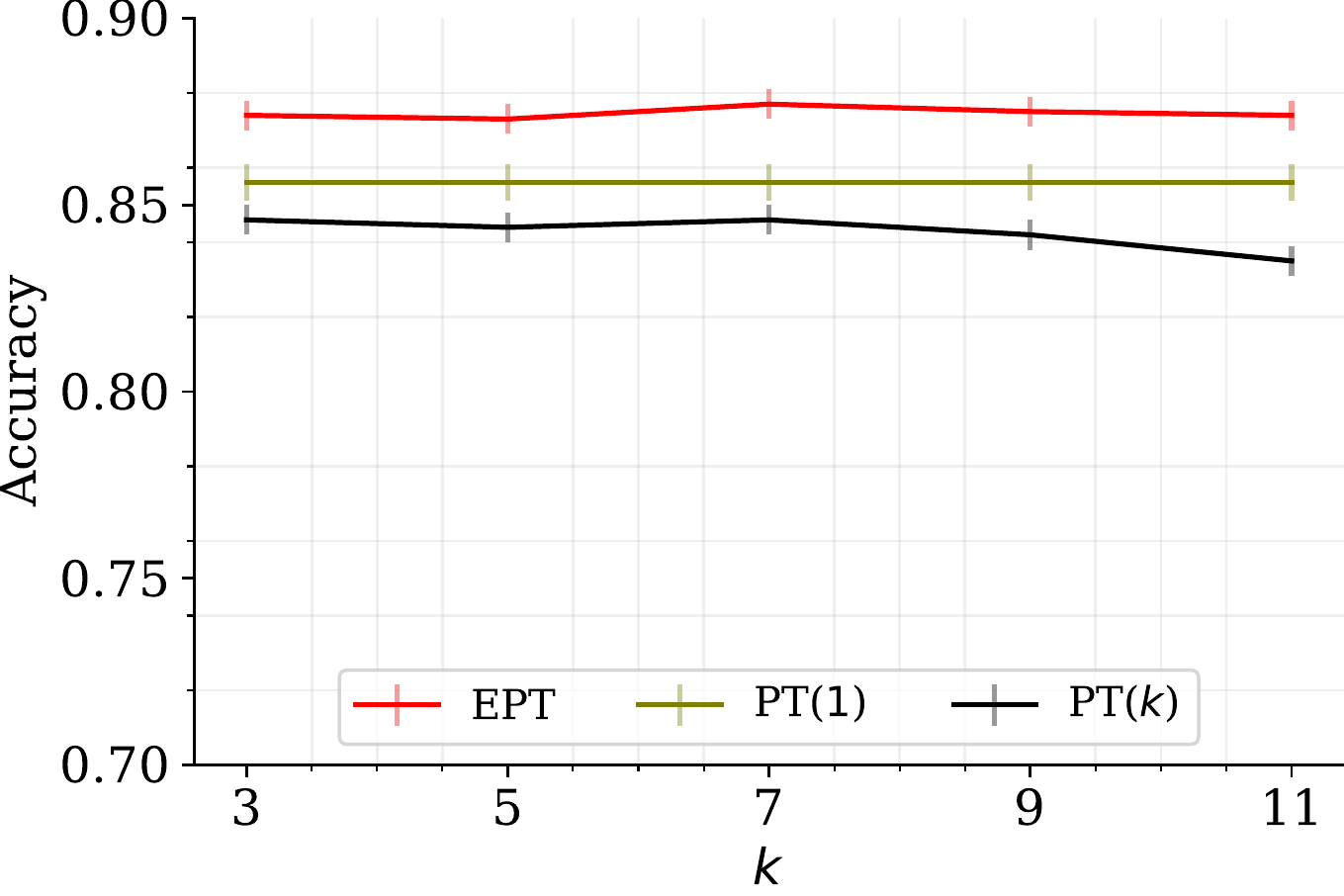}
        \caption{EMBER-B}
    \end{subfigure}%
      \begin{subfigure}{.25\textwidth}
     \includegraphics[scale=0.31]{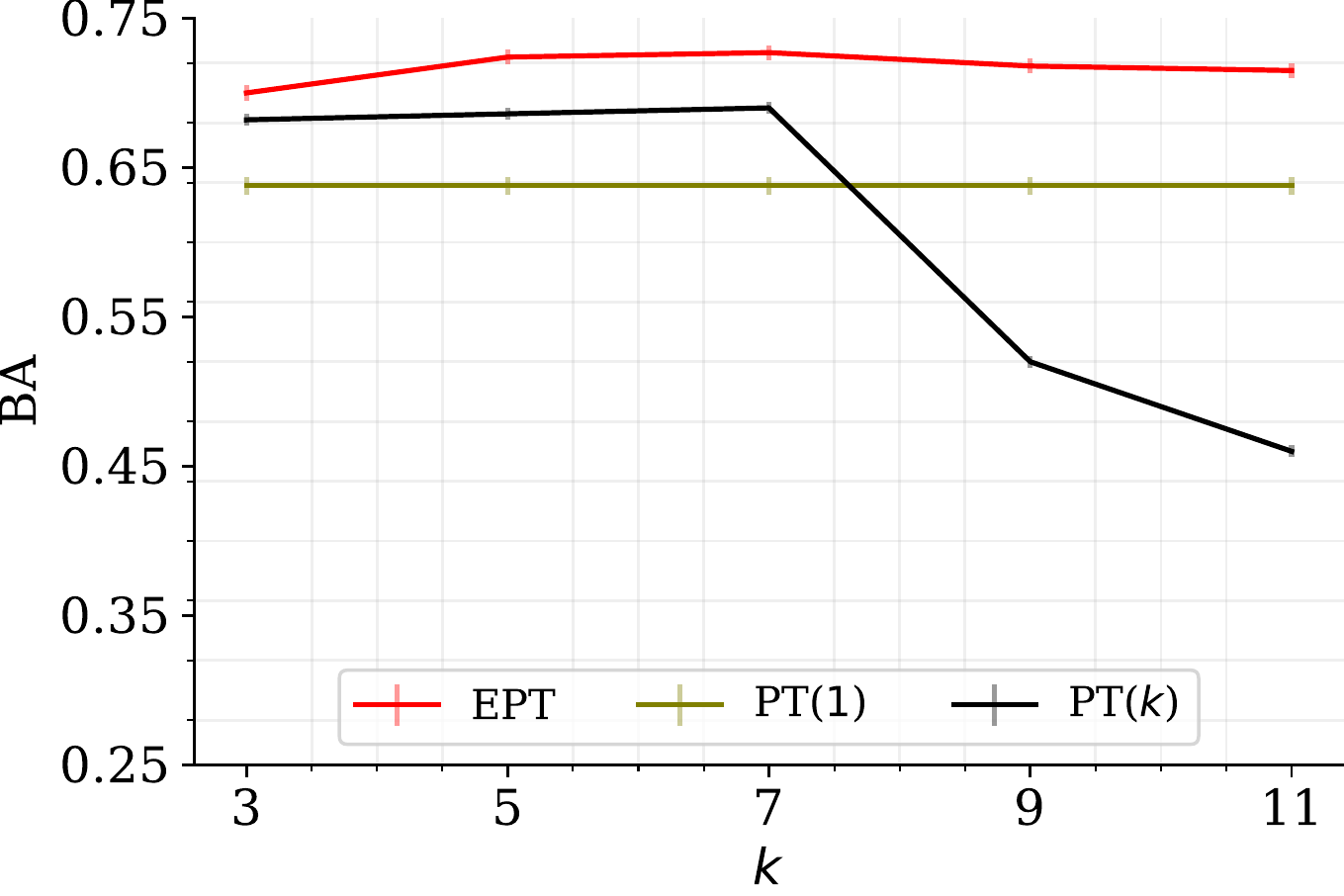}
        \caption{EMBER-U}
    \end{subfigure}%
    \begin{subfigure}{.25\textwidth}
     \includegraphics[scale=0.31]{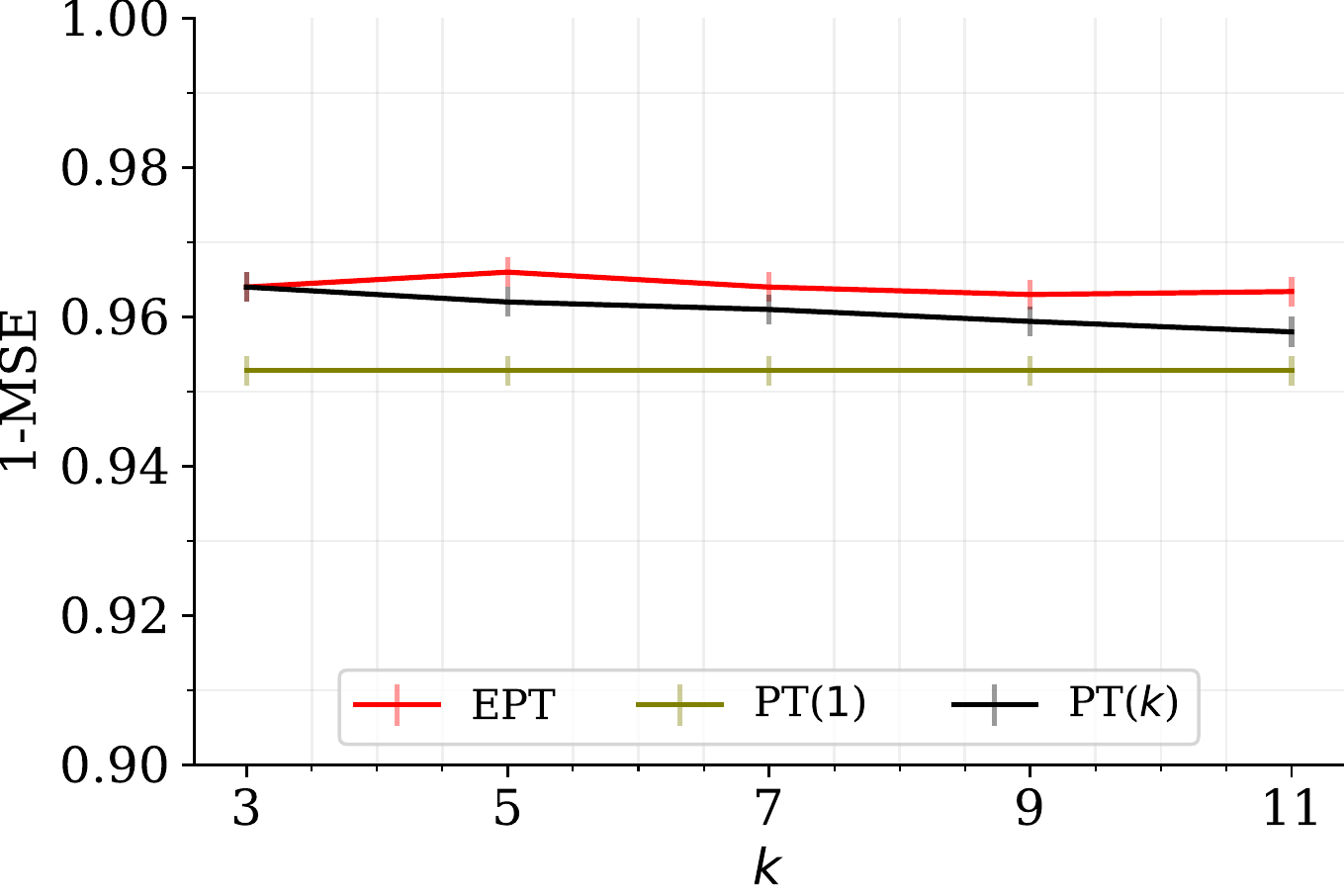}
        \caption{Housing Market}
    \end{subfigure}% 
    \vspace{-0.2cm}
    \caption{\textit{\textbf{Impact of ensemble}} (Comparing EPT vs PT($1$) and PT($k$)).
First row shows the comparison with varying $\varepsilon$ while the second row shows the comparison with varying $k$. }\label{fig:ab_ens}
\vspace{-0.2cm}
\end{figure*}

\begin{figure*}[]
 \centering
    \begin{subfigure}{.25\textwidth}
     \includegraphics[scale=0.31]{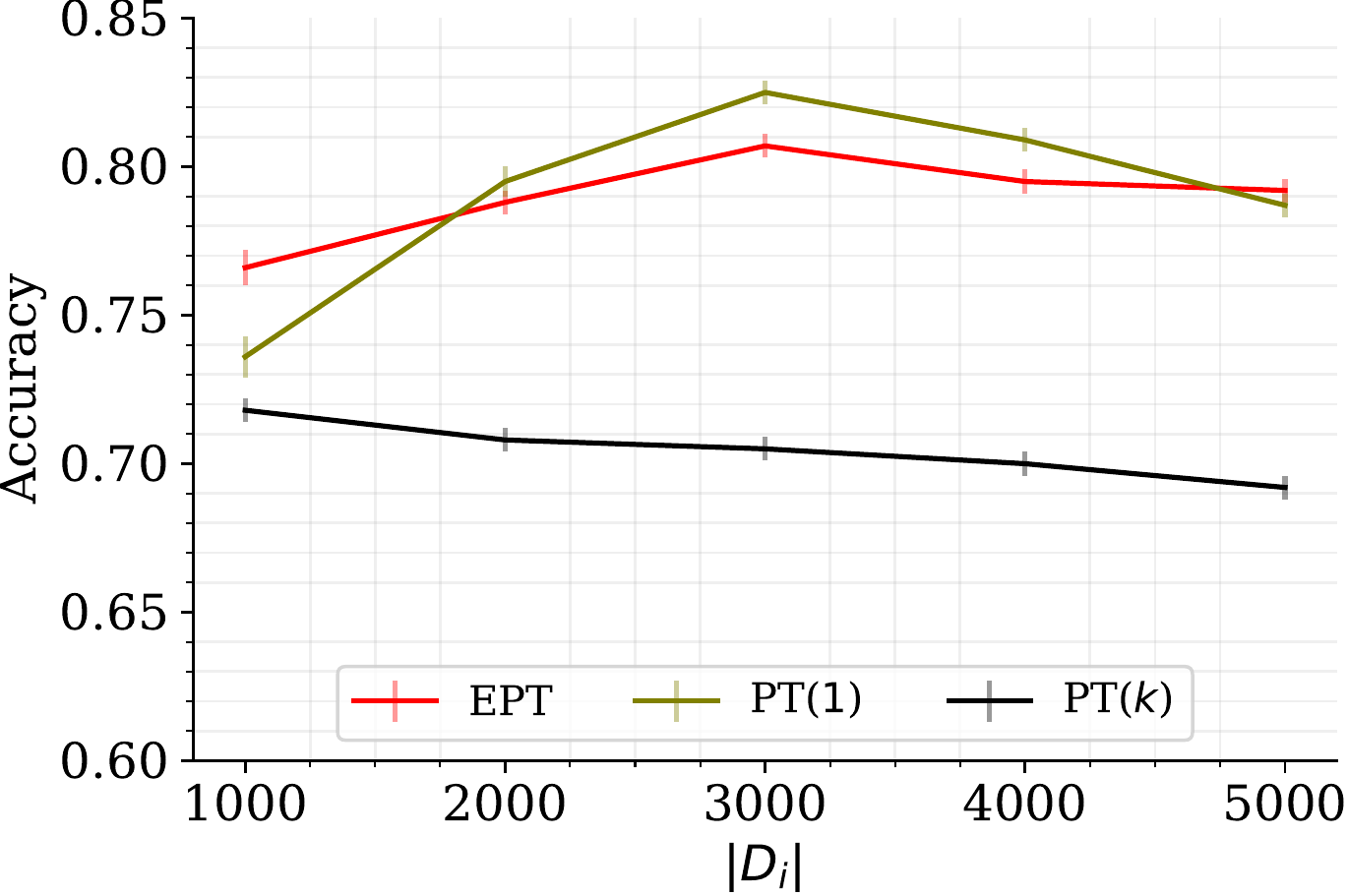}
        \caption{Hyperplane (Rapid)}
    \end{subfigure}%
      \begin{subfigure}{.25\textwidth}
     \includegraphics[scale=0.31]{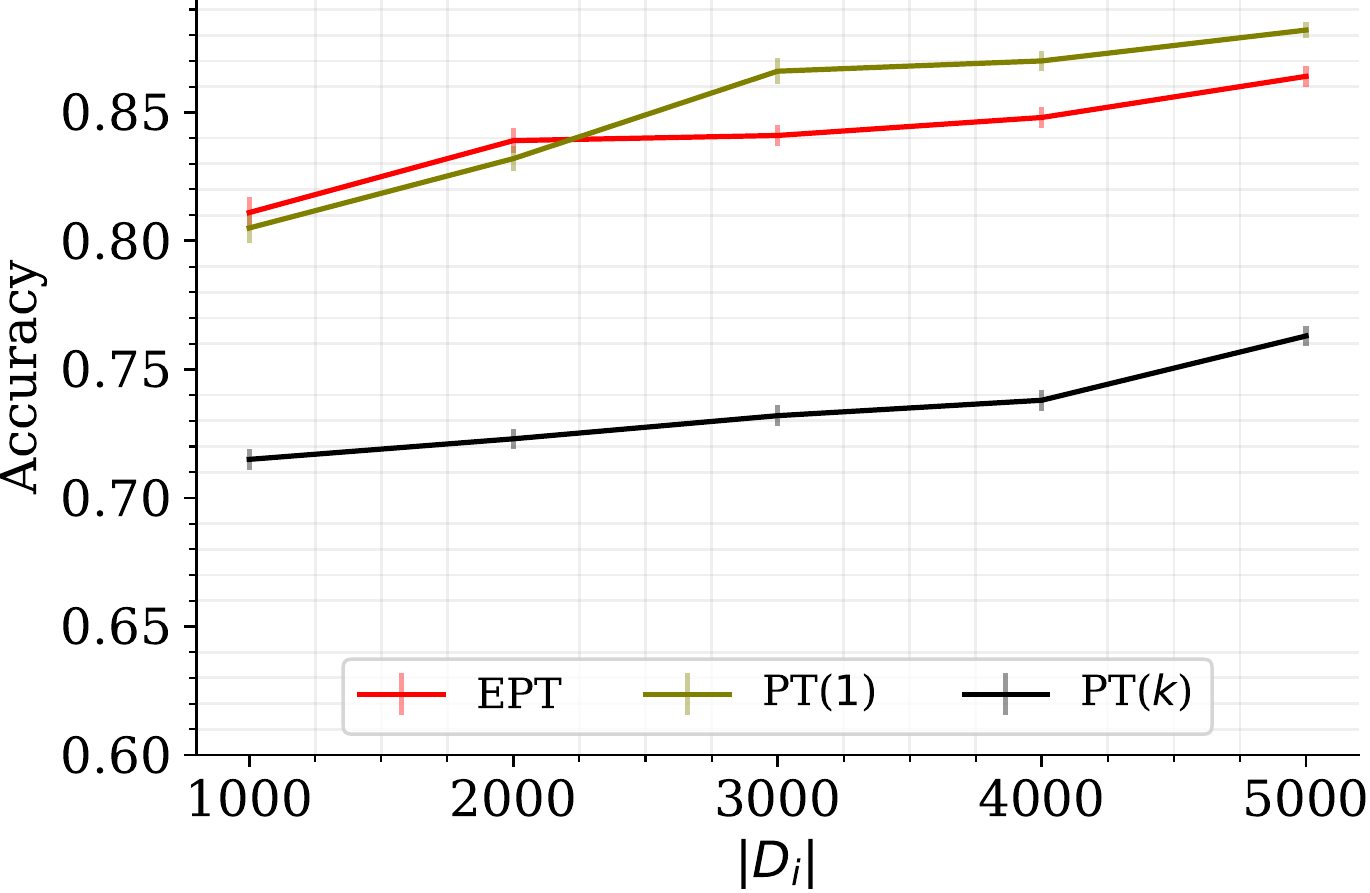}
        \caption{Hyperplane (Gradual)}
    \end{subfigure}%
        \begin{subfigure}{.25\textwidth}
     \includegraphics[scale=0.31]{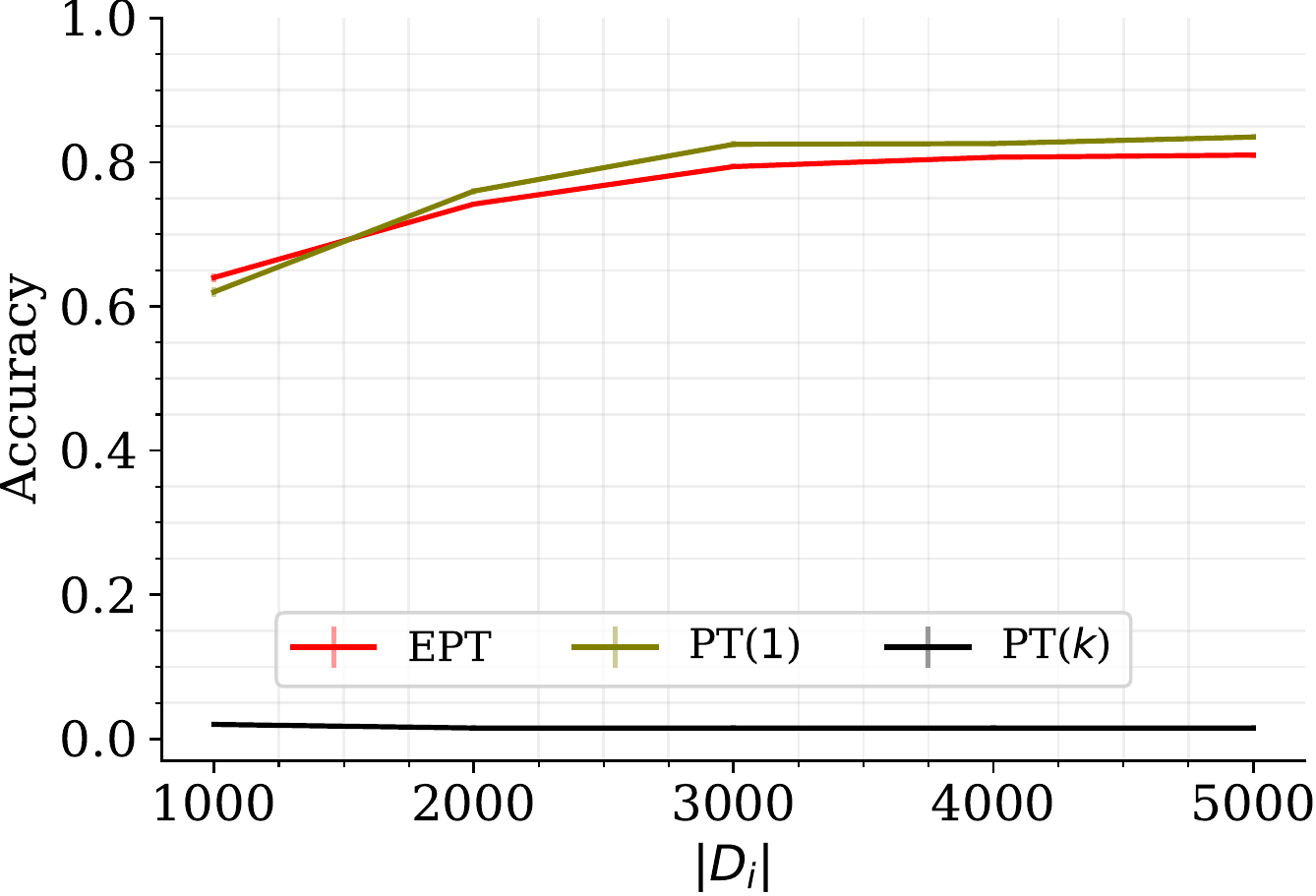}
        \caption{Hyperplane (Abrupt)}
    \end{subfigure}% 
    \begin{subfigure}{.25\textwidth}
     \includegraphics[scale=0.31]{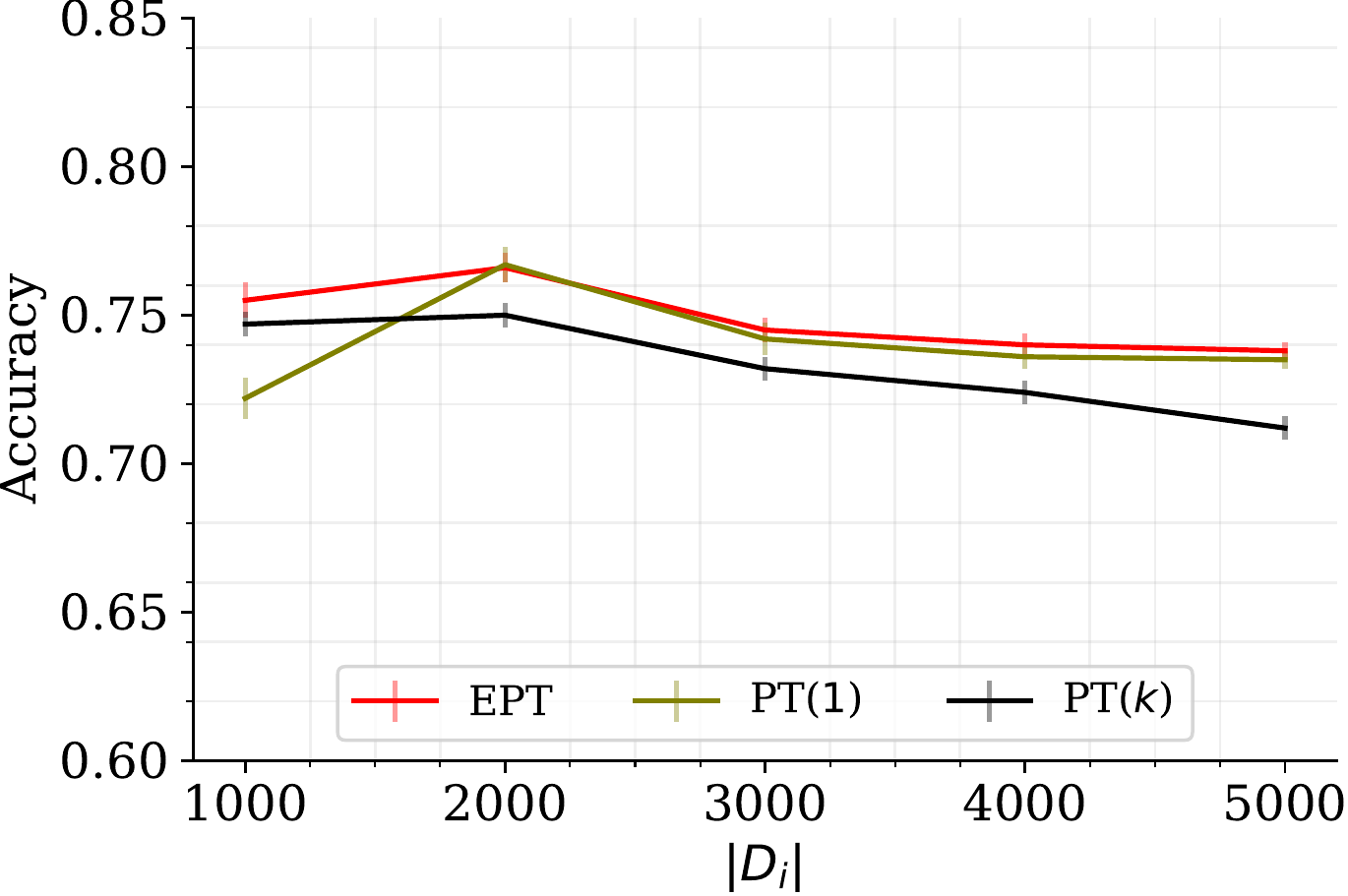}
        \caption{Hyperplane (Recurring)}
    \end{subfigure}% 
    \vspace{-0.2cm}
    \caption{\textit{\textbf{Impact of drifts}} (Comparing  ensemble approaches for various drift types and chunk sizes.)}\label{fig:drift}
    \vspace{-0.2cm}
\end{figure*}

\begin{figure*}[]
 \centering
             \begin{subfigure}{.25\textwidth}
     \includegraphics[scale=0.31]{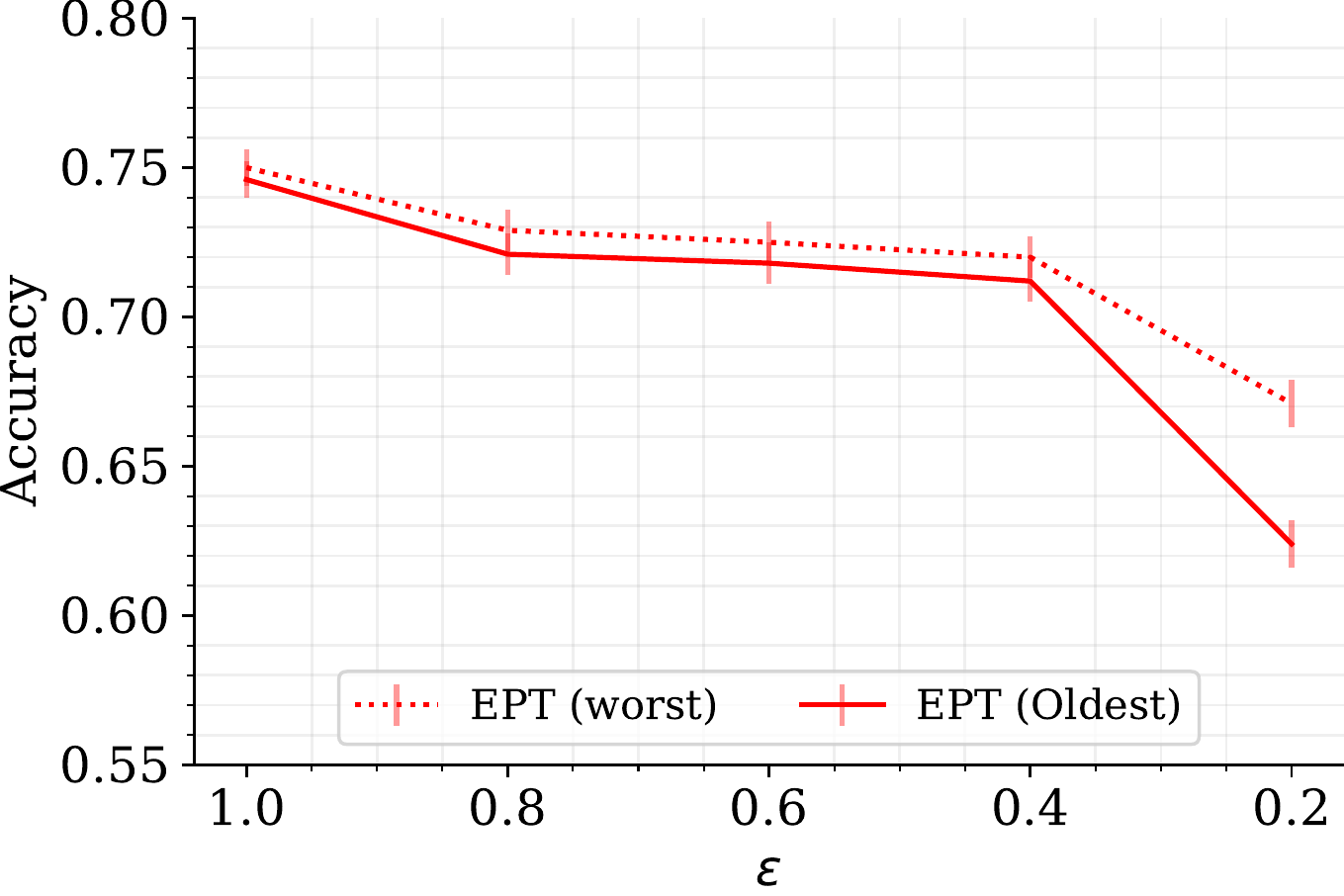}
        \caption{Hyperplane (Recurring)}
    \end{subfigure}% 
    \begin{subfigure}{.25\textwidth}
     \includegraphics[scale=0.31]{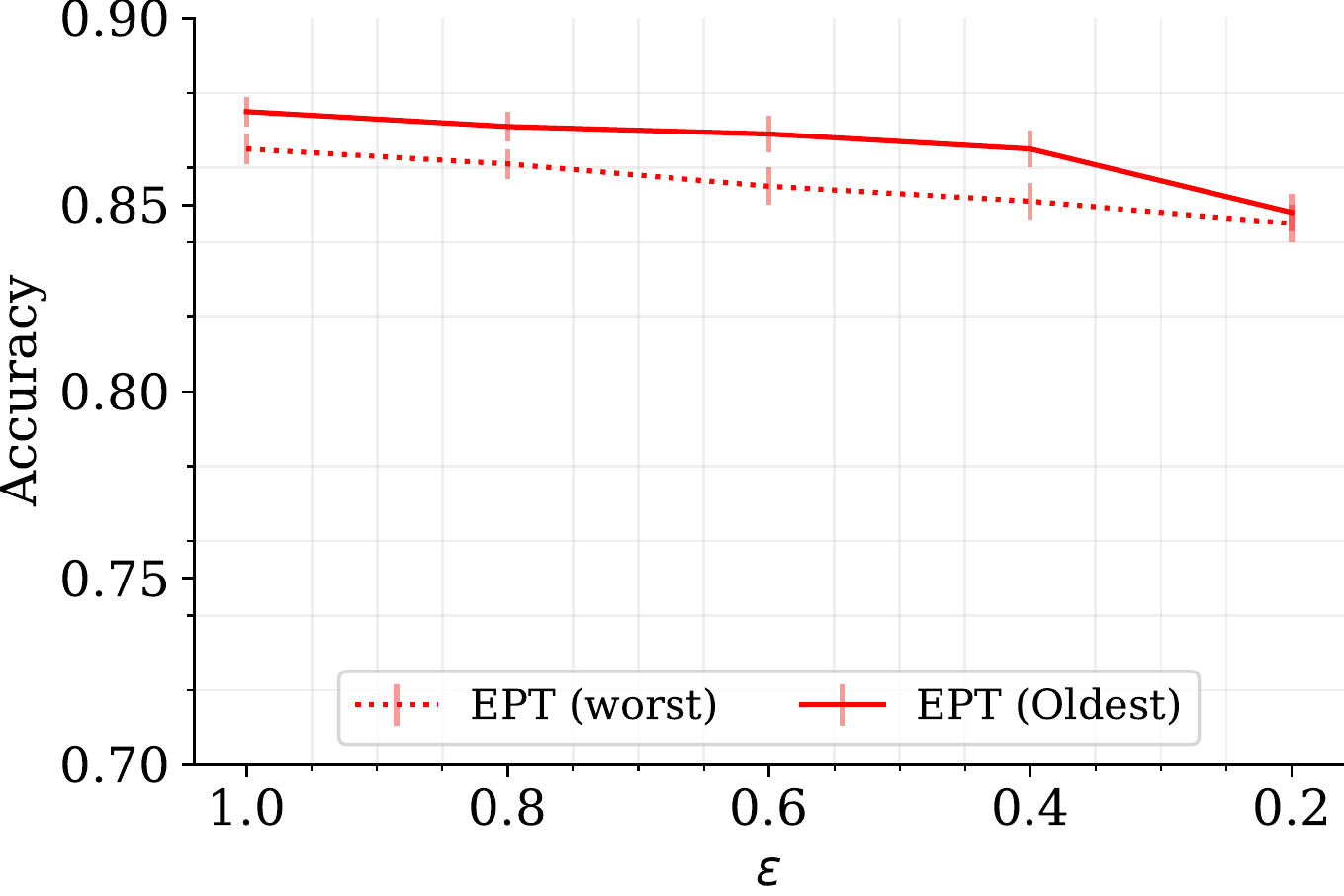}
        \caption{EMBER-B}
    \end{subfigure}%
      \begin{subfigure}{.25\textwidth}
     \includegraphics[scale=0.31]{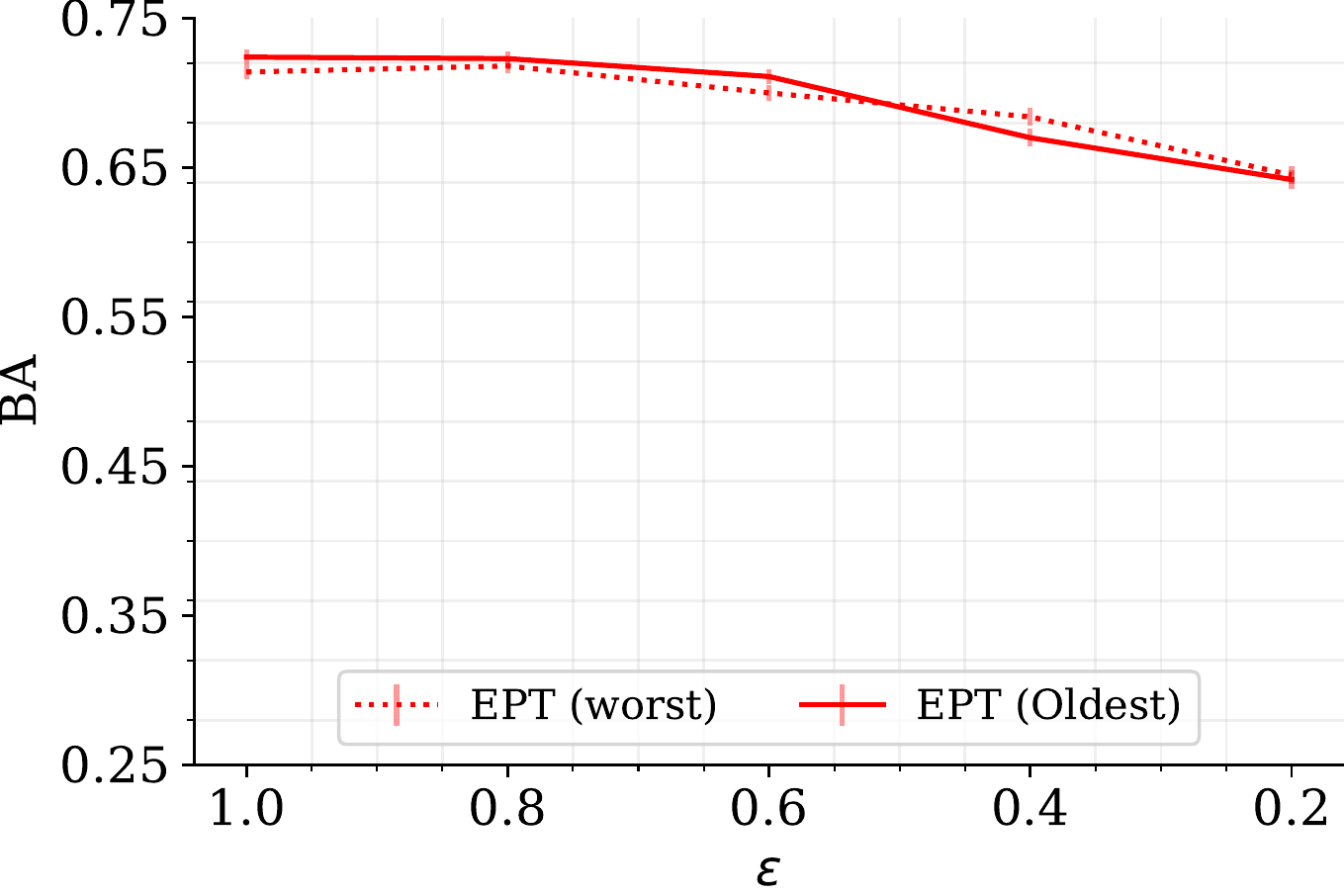}
        \caption{EMBER-U}
    \end{subfigure}%
    \begin{subfigure}{.25\textwidth}
     \includegraphics[scale=0.31]{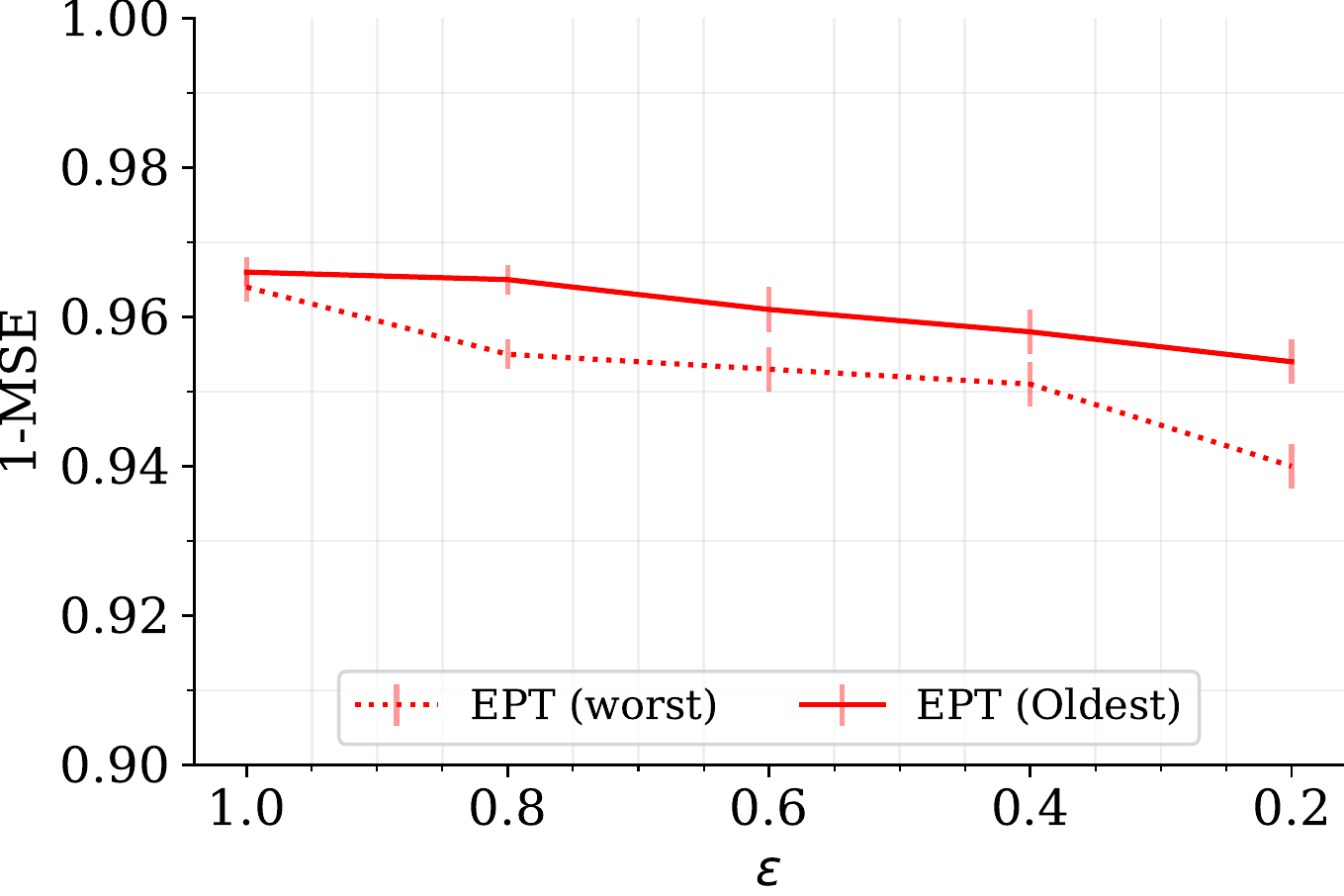}
        \caption{Housing Market}
    \end{subfigure}% 
    \vspace{-0.2cm}
    \caption{\textit{\textbf{Impact of model replacement}} (Comparing replacing oldest model vs replacing worst model).}\label{fig:worsevsold}
    \vspace{-0.2cm}
\end{figure*}

From now, we consider only the privacy preserving methods, i.e., EPT, EP, PT($1$), and PT($k$).

\subsection{Impact of Transfer Learning}\label{sec:transfer}
Figure \ref{fig:ab_trans} studies the effect of transfer learning by comparing EPT against the non-transfer learning counterpart EP (Note that EPT and EP are same for Hyperplane that has no transfer learning). 
EPT outperforms EP by a \emph{significant} margin, with the average boost $>10\%$ for EMBER-B and EMBER-U, and 7\% for Housing Market, for all settings of $\varepsilon$. As the privacy budget gets tighter ($\varepsilon$ decreases), EP decays in performance faster than EPT. This study supports our claim 
at the beginning of this section that transfer learning can boost the utility for training private models by reducing the number of trainable parameters for DPSGD.
% especially useful for smaller chunk sizes. 

From now, we consider only DP methods with transfer learning, i.e., EPT, PT($1$), and PT($k$).

\subsection{Impact of Ensemble}\label{sec:ensemble_k}
To investigate how the ensemble approach helps, in Figure \ref{fig:ab_ens} we compare EPT with the non-ensemble counterparts PT($1$) and PT($k$). The first row varies $\varepsilon$ and the second row varies $k$. With the varying $\varepsilon$, EPT outperforms the non-ensemble competitors consistently because our novel DP weight mechanisms diminish the weights for outdated models. PT($k$), in general, performs better than PT($1$), except for EMBER-B, because PT($1$) uses a single data chunk, which leads to a larger $\sigma$ for the Gaussian noise, as discussed at the beginning of this section. 
When $k\geq 5$, there is some performance decline for EPT because more outdated data chunks are used in an ensemble and because the budget $\varepsilon_2=\varepsilon/k$ for weight estimation gets tighter, but this decline is smaller than that for PT($k$) because of the ``auto-correction" due to the weighting scheme in EPT. In our evaluation, we observed $3 \ge k \le 7$ perform the best.

\subsection{Impact of Concept Drifts}\label{sec:ablation_drift}
Figure \ref{fig:drift} shows the impact of four simulated drift types using Hyperplane.
When drift is rapid, as the chunk size $|D_i|$ increases, the performance of EPT and PT($1$) initially increases and then decreases due to increasing drift introduced within a data chunk. This decline trend is especially observed for PT($k$) that uses the union of $k$ chunks to train the model. So the chunk size is a double-edged sword for rapid drift: too small or too large will hurt. There is a similar trend for recurring drift. When drift is gradual, all methods benefit as $|D_i|$ increases because drift is introduced slowly. When drift is abrupt, both EPT and PT($1$) adapt well, but PT($k$) fails to learn in this case as it uses a \emph{stale} model, i.e., an abrupt change occurs after the model training.  

% a model trained before an abrupt change occurs was used to predict for the $k$ time points after the abrupt change occurs. 

\subsection{Impact of Model Replacement}\label{sec:newvsold}
Figure \ref{fig:worsevsold} shows the impact of replacing old model vs replacing worst model during the ensemble update of EPT. For Hyperplane, where we choose recurring drift,  replacing the worst model is better than replacing the oldest model. For the other datasets, however, replacing the oldest model gives slightly better performance due to adding a smaller noise in the weight estimation, i.e., $\varepsilon_2=\nicefrac{\varepsilon}{k}$ vs $\varepsilon_2=\nicefrac{\varepsilon}{k+1}$. See Theorem \ref{thm:dp-update-old} and Theorem \ref{thm:dp-update-worse}.

\section{Conclusion}\label{sec:conclusion}
We presented a practical DP solution to predictive modeling (both classification and regression) for data streams with concept drift. To the best of our knowledge, this is the first work that allows an unbounded number of updates under a fixed privacy budget. The key component is a novel DP weighting mechanism for integrating the models in an ensemble. Our solution is model agnostic and can be used with any existing DP classification/regression method. 

% With a large chunk size (e.g., $\geq 2000$) and a reasonable privacy budget (e.g., $\varepsilon=1$), the non-ensemble PT($1$) performs competitively, but the ensemble EPT performs the best overall, especially for small chunk sizes and tighter privacy budgets.

\section*{Acknowledgements}
This research is in part supported by a CGS-D award for Lovedeep Gondara and a discovery grant for Ke Wang from Natural Sciences and Engineering Research Council of Canada.

% Extensive empirical studies show that the proposed method is highly effective for varying privacy settings, concept drifts, and data distributions. 

%\clearpage
\bibliographystyle{ACM-Reference-Format}
\bibliography{sample-base}


\begin{thebibliography}{45}

%%% the Web address.

\ifx \showCODEN    \undefined \def \showCODEN     #1{\unskip}     \fi
\ifx \showDOI      \undefined \def \showDOI       #1{#1}\fi
\ifx \showISBNx    \undefined \def \showISBNx     #1{\unskip}     \fi
\ifx \showISBNxiii \undefined \def \showISBNxiii  #1{\unskip}     \fi
\ifx \showISSN     \undefined \def \showISSN      #1{\unskip}     \fi
\ifx \showLCCN     \undefined \def \showLCCN      #1{\unskip}     \fi
\ifx \shownote     \undefined \def \shownote      #1{#1}          \fi
\ifx \showarticletitle \undefined \def \showarticletitle #1{#1}   \fi
\ifx \showURL      \undefined \def \showURL       {\relax}        \fi
% invisible to TeX
\providecommand\bibfield[2]{#2}
\providecommand\bibinfo[2]{#2}
\providecommand\natexlab[1]{#1}
\providecommand\showeprint[2][]{arXiv:#2}

\bibitem[\protect\citeauthoryear{Abadi, Chu, Goodfellow, McMahan, Mironov,
  Talwar, and Zhang}{Abadi et~al\mbox{.}}{2016}]%
        {abadi2016deep}
\bibfield{author}{\bibinfo{person}{Martin Abadi}, \bibinfo{person}{Andy Chu},
  \bibinfo{person}{Ian Goodfellow}, \bibinfo{person}{H~Brendan McMahan},
  \bibinfo{person}{Ilya Mironov}, \bibinfo{person}{Kunal Talwar}, {and}
  \bibinfo{person}{Li Zhang}.} \bibinfo{year}{2016}\natexlab{}.
\newblock \showarticletitle{Deep learning with differential privacy}. In
  \bibinfo{booktitle}{\emph{Proceedings of the 2016 ACM SIGSAC Conference on
  Computer and Communications Security}}. ACM, \bibinfo{pages}{308--318}.
\newblock


\bibitem[\protect\citeauthoryear{Allix, Bissyand{\'e}, Klein, and
  Le~Traon}{Allix et~al\mbox{.}}{2016}]%
        {Allix:2016:ACM:2901739.2903508}
\bibfield{author}{\bibinfo{person}{Kevin Allix},
  \bibinfo{person}{Tegawend{\'e}~F. Bissyand{\'e}}, \bibinfo{person}{Jacques
  Klein}, {and} \bibinfo{person}{Yves Le~Traon}.}
  \bibinfo{year}{2016}\natexlab{}.
\newblock \showarticletitle{AndroZoo: Collecting Millions of Android Apps for
  the Research Community}. In \bibinfo{booktitle}{\emph{Proceedings of the 13th
  International Conference on Mining Software Repositories}} (Austin, Texas)
  \emph{(\bibinfo{series}{MSR '16})}. \bibinfo{publisher}{ACM},
  \bibinfo{address}{New York, NY, USA}, \bibinfo{pages}{468--471}.
\newblock
\showISBNx{978-1-4503-4186-8}
\urldef\tempurl%
\url{https://doi.org/10.1145/2901739.2903508}
\showDOI{\tempurl}


\bibitem[\protect\citeauthoryear{{Anderson} and {Roth}}{{Anderson} and
  {Roth}}{2018}]%
        {2018arXiv180404637A}
\bibfield{author}{\bibinfo{person}{H.~S. {Anderson}} {and} \bibinfo{person}{P.
  {Roth}}.} \bibinfo{year}{2018}\natexlab{}.
\newblock \showarticletitle{{EMBER: An Open Dataset for Training Static PE
  Malware Machine Learning Models}}.
\newblock \bibinfo{journal}{\emph{ArXiv e-prints}} (\bibinfo{date}{April}
  \bibinfo{year}{2018}).
\newblock
\showeprint[arxiv]{1804.04637}~[cs.CR]


\bibitem[\protect\citeauthoryear{Barthe, Farina, Gaboardi, Arias, Gordon, Hsu,
  and Strub}{Barthe et~al\mbox{.}}{2016}]%
        {barthe2016differentially}
\bibfield{author}{\bibinfo{person}{Gilles Barthe}, \bibinfo{person}{Gian~Pietro
  Farina}, \bibinfo{person}{Marco Gaboardi}, \bibinfo{person}{Emilio
  Jes{\'u}s~Gallego Arias}, \bibinfo{person}{Andy Gordon},
  \bibinfo{person}{Justin Hsu}, {and} \bibinfo{person}{Pierre-Yves Strub}.}
  \bibinfo{year}{2016}\natexlab{}.
\newblock \showarticletitle{Differentially private bayesian programming}. In
  \bibinfo{booktitle}{\emph{Proceedings of the 2016 ACM SIGSAC Conference on
  Computer and Communications Security}}. \bibinfo{pages}{68--79}.
\newblock


\bibitem[\protect\citeauthoryear{Brodersen, Ong, Stephan, and
  Buhmann}{Brodersen et~al\mbox{.}}{2010}]%
        {brodersen2010balanced}
\bibfield{author}{\bibinfo{person}{Kay~Henning Brodersen},
  \bibinfo{person}{Cheng~Soon Ong}, \bibinfo{person}{Klaas~Enno Stephan}, {and}
  \bibinfo{person}{Joachim~M Buhmann}.} \bibinfo{year}{2010}\natexlab{}.
\newblock \showarticletitle{The balanced accuracy and its posterior
  distribution}. In \bibinfo{booktitle}{\emph{2010 20th international
  conference on pattern recognition}}. IEEE, \bibinfo{pages}{3121--3124}.
\newblock


\bibitem[\protect\citeauthoryear{Brzezi{\'n}ski and Stefanowski}{Brzezi{\'n}ski
  and Stefanowski}{2011}]%
        {brzezinski2011accuracy}
\bibfield{author}{\bibinfo{person}{Dariusz Brzezi{\'n}ski} {and}
  \bibinfo{person}{Jerzy Stefanowski}.} \bibinfo{year}{2011}\natexlab{}.
\newblock \showarticletitle{Accuracy updated ensemble for data streams with
  concept drift}. In \bibinfo{booktitle}{\emph{International conference on
  hybrid artificial intelligence systems}}. Springer,
  \bibinfo{pages}{155--163}.
\newblock


\bibitem[\protect\citeauthoryear{Cano and Krawczyk}{Cano and Krawczyk}{2020}]%
        {cano2020kappa}
\bibfield{author}{\bibinfo{person}{Alberto Cano} {and} \bibinfo{person}{Bartosz
  Krawczyk}.} \bibinfo{year}{2020}\natexlab{}.
\newblock \showarticletitle{Kappa updated ensemble for drifting data stream
  mining}.
\newblock \bibinfo{journal}{\emph{Machine Learning}} \bibinfo{volume}{109},
  \bibinfo{number}{1} (\bibinfo{year}{2020}), \bibinfo{pages}{175--218}.
\newblock


\bibitem[\protect\citeauthoryear{CIHI}{CIHI}{2020}]%
        {cihidad}
\bibfield{author}{\bibinfo{person}{CIHI}.} \bibinfo{year}{2020}\natexlab{}.
\newblock \bibinfo{title}{Discharge Abstract Database}.
\newblock
\newblock
\newblock
\shownote{\url{https://www.cihi.ca/en/discharge-abstract-database-metadata}.}


\bibitem[\protect\citeauthoryear{Dal~Pozzolo, Boracchi, Caelen, Alippi, and
  Bontempi}{Dal~Pozzolo et~al\mbox{.}}{2015}]%
        {dal2015credit}
\bibfield{author}{\bibinfo{person}{Andrea Dal~Pozzolo},
  \bibinfo{person}{Giacomo Boracchi}, \bibinfo{person}{Olivier Caelen},
  \bibinfo{person}{Cesare Alippi}, {and} \bibinfo{person}{Gianluca Bontempi}.}
  \bibinfo{year}{2015}\natexlab{}.
\newblock \showarticletitle{Credit card fraud detection and concept-drift
  adaptation with delayed supervised information}. In
  \bibinfo{booktitle}{\emph{2015 international joint conference on Neural
  networks (IJCNN)}}. IEEE, \bibinfo{pages}{1--8}.
\newblock


\bibitem[\protect\citeauthoryear{De~Francisci~Morales, Bifet, Khan, Gama, and
  Fan}{De~Francisci~Morales et~al\mbox{.}}{2016}]%
        {de2016iot}
\bibfield{author}{\bibinfo{person}{Gianmarco De~Francisci~Morales},
  \bibinfo{person}{Albert Bifet}, \bibinfo{person}{Latifur Khan},
  \bibinfo{person}{Joao Gama}, {and} \bibinfo{person}{Wei Fan}.}
  \bibinfo{year}{2016}\natexlab{}.
\newblock \showarticletitle{Iot big data stream mining}. In
  \bibinfo{booktitle}{\emph{Proceedings of the 22nd ACM SIGKDD international
  conference on knowledge discovery and data mining}}.
  \bibinfo{pages}{2119--2120}.
\newblock


\bibitem[\protect\citeauthoryear{de~Lima~Cabral and de~Barros}{de~Lima~Cabral
  and de~Barros}{2018}]%
        {de2018ddmfishers}
\bibfield{author}{\bibinfo{person}{Danilo~Rafael de Lima~Cabral} {and}
  \bibinfo{person}{Roberto Souto~Maior de Barros}.}
  \bibinfo{year}{2018}\natexlab{}.
\newblock \showarticletitle{Concept drift detection based on Fisher’s Exact
  test}.
\newblock \bibinfo{journal}{\emph{Information Sciences}}  \bibinfo{volume}{442}
  (\bibinfo{year}{2018}), \bibinfo{pages}{220--234}.
\newblock


\bibitem[\protect\citeauthoryear{Dwork, Kenthapadi, McSherry, Mironov, and
  Naor}{Dwork et~al\mbox{.}}{2006a}]%
        {dwork2006our}
\bibfield{author}{\bibinfo{person}{Cynthia Dwork}, \bibinfo{person}{Krishnaram
  Kenthapadi}, \bibinfo{person}{Frank McSherry}, \bibinfo{person}{Ilya
  Mironov}, {and} \bibinfo{person}{Moni Naor}.}
  \bibinfo{year}{2006}\natexlab{a}.
\newblock \showarticletitle{Our data, ourselves: privacy via distributed noise
  generation}.
\newblock \bibinfo{journal}{\emph{EUROCRYPT}} (\bibinfo{year}{2006}),
  \bibinfo{pages}{486--503}.
\newblock


\bibitem[\protect\citeauthoryear{Dwork, McSherry, Nissim, and Smith}{Dwork
  et~al\mbox{.}}{2006b}]%
        {Dwork:2006:CNS:2180286.2180305}
\bibfield{author}{\bibinfo{person}{Cynthia Dwork}, \bibinfo{person}{Frank
  McSherry}, \bibinfo{person}{Kobbi Nissim}, {and} \bibinfo{person}{Adam
  Smith}.} \bibinfo{year}{2006}\natexlab{b}.
\newblock \showarticletitle{Calibrating Noise to Sensitivity in Private Data
  Analysis}. In \bibinfo{booktitle}{\emph{Proceedings of the Third Conference
  on Theory of Cryptography}} (New York, NY) \emph{(\bibinfo{series}{TCC'06})}.
  \bibinfo{publisher}{Springer-Verlag}, \bibinfo{address}{Berlin, Heidelberg},
  \bibinfo{pages}{265--284}.
\newblock
\showISBNx{3-540-32731-2, 978-3-540-32731-8}
\urldef\tempurl%
\url{https://doi.org/10.1007/11681878_14}
\showDOI{\tempurl}


\bibitem[\protect\citeauthoryear{Dwork, Naor, Pitassi, Rothblum, and
  Yekhanin}{Dwork et~al\mbox{.}}{2010}]%
        {dwork2010pan}
\bibfield{author}{\bibinfo{person}{Cynthia Dwork}, \bibinfo{person}{Moni Naor},
  \bibinfo{person}{Toniann Pitassi}, \bibinfo{person}{Guy~N Rothblum}, {and}
  \bibinfo{person}{Sergey Yekhanin}.} \bibinfo{year}{2010}\natexlab{}.
\newblock \showarticletitle{Pan-Private Streaming Algorithms.}. In
  \bibinfo{booktitle}{\emph{ICS}}. \bibinfo{pages}{66--80}.
\newblock


\bibitem[\protect\citeauthoryear{Dwork, Roth, et~al\mbox{.}}{Dwork
  et~al\mbox{.}}{2014}]%
        {dwork2014book}
\bibfield{author}{\bibinfo{person}{Cynthia Dwork}, \bibinfo{person}{Aaron
  Roth}, {et~al\mbox{.}}} \bibinfo{year}{2014}\natexlab{}.
\newblock \showarticletitle{The algorithmic foundations of differential
  privacy}.
\newblock \bibinfo{journal}{\emph{Foundations and Trends{\textregistered} in
  Theoretical Computer Science}} (\bibinfo{year}{2014}).
\newblock


\bibitem[\protect\citeauthoryear{Ebada, Abdelrazek, and Elhenawy}{Ebada
  et~al\mbox{.}}{2020}]%
        {ebada2020applying}
\bibfield{author}{\bibinfo{person}{Ahmed~Ismail Ebada}, \bibinfo{person}{Samir
  Abdelrazek}, {and} \bibinfo{person}{Ibrahim Elhenawy}.}
  \bibinfo{year}{2020}\natexlab{}.
\newblock \showarticletitle{Applying Cloud Based Machine Learning on Biosensors
  Streaming Data for Health Status Prediction}. In
  \bibinfo{booktitle}{\emph{2020 11th International Conference on Information,
  Intelligence, Systems and Applications (IISA}}. IEEE, \bibinfo{pages}{1--8}.
\newblock


\bibitem[\protect\citeauthoryear{Elwell and Polikar}{Elwell and
  Polikar}{2011}]%
        {elwell2011incremental}
\bibfield{author}{\bibinfo{person}{Ryan Elwell} {and} \bibinfo{person}{Robi
  Polikar}.} \bibinfo{year}{2011}\natexlab{}.
\newblock \showarticletitle{Incremental learning of concept drift in
  nonstationary environments}.
\newblock \bibinfo{journal}{\emph{IEEE Transactions on Neural Networks}}
  \bibinfo{volume}{22}, \bibinfo{number}{10} (\bibinfo{year}{2011}),
  \bibinfo{pages}{1517--1531}.
\newblock


\bibitem[\protect\citeauthoryear{Fanaeepour and Machanavajjhala}{Fanaeepour and
  Machanavajjhala}{2019}]%
        {fanaeepour2019privstream}
\bibfield{author}{\bibinfo{person}{Maryam Fanaeepour} {and}
  \bibinfo{person}{Ashwin Machanavajjhala}.} \bibinfo{year}{2019}\natexlab{}.
\newblock \showarticletitle{PrivStream: differentially private event detection
  on data streams}. In \bibinfo{booktitle}{\emph{Proceedings of the Ninth ACM
  Conference on Data and Application Security and Privacy}}.
  \bibinfo{pages}{145--147}.
\newblock


\bibitem[\protect\citeauthoryear{Fang, Zeng, and Yang}{Fang
  et~al\mbox{.}}{2020}]%
        {fang2020localdp}
\bibfield{author}{\bibinfo{person}{Xianjin Fang}, \bibinfo{person}{Qingkui
  Zeng}, {and} \bibinfo{person}{Gaoming Yang}.}
  \bibinfo{year}{2020}\natexlab{}.
\newblock \showarticletitle{Local Differential Privacy for Data Streams}. In
  \bibinfo{booktitle}{\emph{International Conference on Security and Privacy in
  Digital Economy}}. Springer, \bibinfo{pages}{143--160}.
\newblock


\bibitem[\protect\citeauthoryear{Gama, {\v{Z}}liobait{\.e}, Bifet, Pechenizkiy,
  and Bouchachia}{Gama et~al\mbox{.}}{2014}]%
        {gama2014survey}
\bibfield{author}{\bibinfo{person}{Jo{\~a}o Gama}, \bibinfo{person}{Indr{\.e}
  {\v{Z}}liobait{\.e}}, \bibinfo{person}{Albert Bifet}, \bibinfo{person}{Mykola
  Pechenizkiy}, {and} \bibinfo{person}{Abdelhamid Bouchachia}.}
  \bibinfo{year}{2014}\natexlab{}.
\newblock \showarticletitle{A survey on concept drift adaptation}.
\newblock \bibinfo{journal}{\emph{ACM computing surveys (CSUR)}}
  \bibinfo{volume}{46}, \bibinfo{number}{4} (\bibinfo{year}{2014}),
  \bibinfo{pages}{1--37}.
\newblock


\bibitem[\protect\citeauthoryear{G{\"o}z{\"u}a{\c{c}}{\i}k,
  B{\"u}y{\"u}k{\c{c}}ak{\i}r, Bonab, and Can}{G{\"o}z{\"u}a{\c{c}}{\i}k
  et~al\mbox{.}}{2019}]%
        {gozuaccik2019unsupervised}
\bibfield{author}{\bibinfo{person}{{\"O}mer G{\"o}z{\"u}a{\c{c}}{\i}k},
  \bibinfo{person}{Alican B{\"u}y{\"u}k{\c{c}}ak{\i}r}, \bibinfo{person}{Hamed
  Bonab}, {and} \bibinfo{person}{Fazli Can}.} \bibinfo{year}{2019}\natexlab{}.
\newblock \showarticletitle{Unsupervised concept drift detection with a
  discriminative classifier}. In \bibinfo{booktitle}{\emph{Proceedings of the
  28th ACM International Conference on Information and Knowledge Management}}.
  \bibinfo{pages}{2365--2368}.
\newblock


\bibitem[\protect\citeauthoryear{Hulten, Spencer, and Domingos}{Hulten
  et~al\mbox{.}}{2001}]%
        {hulten2001mining}
\bibfield{author}{\bibinfo{person}{Geoff Hulten}, \bibinfo{person}{Laurie
  Spencer}, {and} \bibinfo{person}{Pedro Domingos}.}
  \bibinfo{year}{2001}\natexlab{}.
\newblock \showarticletitle{Mining time-changing data streams}. In
  \bibinfo{booktitle}{\emph{Proceedings of the seventh ACM SIGKDD international
  conference on Knowledge discovery and data mining}}.
  \bibinfo{pages}{97--106}.
\newblock


\bibitem[\protect\citeauthoryear{Jagannathan, Pillaipakkamnatt, and
  Wright}{Jagannathan et~al\mbox{.}}{2009}]%
        {jagannathan2009practical}
\bibfield{author}{\bibinfo{person}{Geetha Jagannathan},
  \bibinfo{person}{Krishnan Pillaipakkamnatt}, {and} \bibinfo{person}{Rebecca~N
  Wright}.} \bibinfo{year}{2009}\natexlab{}.
\newblock \showarticletitle{A practical differentially private random decision
  tree classifier}. In \bibinfo{booktitle}{\emph{2009 IEEE International
  Conference on Data Mining Workshops}}. IEEE, \bibinfo{pages}{114--121}.
\newblock


\bibitem[\protect\citeauthoryear{Jayaraman and Evans}{Jayaraman and
  Evans}{2019}]%
        {jayaraman2019evaluating}
\bibfield{author}{\bibinfo{person}{Bargav Jayaraman} {and}
  \bibinfo{person}{David Evans}.} \bibinfo{year}{2019}\natexlab{}.
\newblock \showarticletitle{Evaluating differentially private machine learning
  in practice}. In \bibinfo{booktitle}{\emph{28th $\{$USENIX$\}$ Security
  Symposium ($\{$USENIX$\}$ Security 19)}}. \bibinfo{pages}{1895--1912}.
\newblock


\bibitem[\protect\citeauthoryear{Kamath and Ullman}{Kamath and Ullman}{2020}]%
        {kamath2020primer}
\bibfield{author}{\bibinfo{person}{Gautam Kamath} {and}
  \bibinfo{person}{Jonathan Ullman}.} \bibinfo{year}{2020}\natexlab{}.
\newblock \showarticletitle{A primer on private statistics}.
\newblock \bibinfo{journal}{\emph{arXiv preprint arXiv:2005.00010}}
  (\bibinfo{year}{2020}).
\newblock


\bibitem[\protect\citeauthoryear{Kellaris, Papadopoulos, Xiao, and
  Papadias}{Kellaris et~al\mbox{.}}{2014}]%
        {kellaris2014differentially}
\bibfield{author}{\bibinfo{person}{Georgios Kellaris}, \bibinfo{person}{Stavros
  Papadopoulos}, \bibinfo{person}{Xiaokui Xiao}, {and}
  \bibinfo{person}{Dimitris Papadias}.} \bibinfo{year}{2014}\natexlab{}.
\newblock \showarticletitle{Differentially private event sequences over
  infinite streams}.
\newblock \bibinfo{journal}{\emph{Proceedings of the VLDB Endowment}}
  \bibinfo{volume}{7}, \bibinfo{number}{12} (\bibinfo{year}{2014}),
  \bibinfo{pages}{1155--1166}.
\newblock


\bibitem[\protect\citeauthoryear{Khavkin and Last}{Khavkin and Last}{2018}]%
        {khavkin2018clusterdp}
\bibfield{author}{\bibinfo{person}{Michael Khavkin} {and} \bibinfo{person}{Mark
  Last}.} \bibinfo{year}{2018}\natexlab{}.
\newblock \showarticletitle{Preserving Differential Privacy and Utility of
  Non-stationary Data Streams}. In \bibinfo{booktitle}{\emph{2018 IEEE
  International Conference on Data Mining Workshops (ICDMW)}}. IEEE,
  \bibinfo{pages}{29--34}.
\newblock


\bibitem[\protect\citeauthoryear{Kolter and Maloof}{Kolter and Maloof}{2007}]%
        {kolter2007dynamic}
\bibfield{author}{\bibinfo{person}{J~Zico Kolter} {and}
  \bibinfo{person}{Marcus~A Maloof}.} \bibinfo{year}{2007}\natexlab{}.
\newblock \showarticletitle{Dynamic weighted majority: An ensemble method for
  drifting concepts}.
\newblock \bibinfo{journal}{\emph{Journal of Machine Learning Research}}
  \bibinfo{volume}{8}, \bibinfo{number}{Dec} (\bibinfo{year}{2007}),
  \bibinfo{pages}{2755--2790}.
\newblock


\bibitem[\protect\citeauthoryear{Ksieniewicz, Zyblewski, Chora{\'s}, Kozik,
  Gie{\l}czyk, and Wo{\'z}niak}{Ksieniewicz et~al\mbox{.}}{2020}]%
        {ksieniewicz2020fake}
\bibfield{author}{\bibinfo{person}{Pawe{\l} Ksieniewicz},
  \bibinfo{person}{Pawe{\l} Zyblewski}, \bibinfo{person}{Micha{\l} Chora{\'s}},
  \bibinfo{person}{Rafa{\l} Kozik}, \bibinfo{person}{Agata Gie{\l}czyk}, {and}
  \bibinfo{person}{Micha{\l} Wo{\'z}niak}.} \bibinfo{year}{2020}\natexlab{}.
\newblock \showarticletitle{Fake news detection from data streams}. In
  \bibinfo{booktitle}{\emph{2020 International Joint Conference on Neural
  Networks (IJCNN)}}. IEEE, \bibinfo{pages}{1--8}.
\newblock


\bibitem[\protect\citeauthoryear{Lu, Liu, Dong, Gu, Gama, and Zhang}{Lu
  et~al\mbox{.}}{2018}]%
        {lu2018cdreview}
\bibfield{author}{\bibinfo{person}{Jie Lu}, \bibinfo{person}{Anjin Liu},
  \bibinfo{person}{Fan Dong}, \bibinfo{person}{Feng Gu}, \bibinfo{person}{Joao
  Gama}, {and} \bibinfo{person}{Guangquan Zhang}.}
  \bibinfo{year}{2018}\natexlab{}.
\newblock \showarticletitle{Learning under concept drift: A review}.
\newblock \bibinfo{journal}{\emph{IEEE Transactions on Knowledge and Data
  Engineering}} \bibinfo{volume}{31}, \bibinfo{number}{12}
  (\bibinfo{year}{2018}), \bibinfo{pages}{2346--2363}.
\newblock


\bibitem[\protect\citeauthoryear{Lukes and S{\o}gaard}{Lukes and
  S{\o}gaard}{2018}]%
        {lukes2018sentiment}
\bibfield{author}{\bibinfo{person}{Jan Lukes} {and} \bibinfo{person}{Anders
  S{\o}gaard}.} \bibinfo{year}{2018}\natexlab{}.
\newblock \showarticletitle{Sentiment analysis under temporal shift}. In
  \bibinfo{booktitle}{\emph{Proceedings of the 9th Workshop on Computational
  Approaches to Subjectivity, Sentiment and Social Media Analysis}}.
  \bibinfo{pages}{65--71}.
\newblock


\bibitem[\protect\citeauthoryear{McSherry}{McSherry}{2009}]%
        {frank2009}
\bibfield{author}{\bibinfo{person}{Frank McSherry}.}
  \bibinfo{year}{2009}\natexlab{}.
\newblock \showarticletitle{Privacy integrated queries: an extensible platform
  for privacy-preserving data analysis}. In \bibinfo{booktitle}{\emph{SIGMOD}}.
  ADM.
\newblock


\bibitem[\protect\citeauthoryear{Montiel, Read, Bifet, and Abdessalem}{Montiel
  et~al\mbox{.}}{2018}]%
        {10.5555/3291125.3309634}
\bibfield{author}{\bibinfo{person}{Jacob Montiel}, \bibinfo{person}{Jesse
  Read}, \bibinfo{person}{Albert Bifet}, {and} \bibinfo{person}{Talel
  Abdessalem}.} \bibinfo{year}{2018}\natexlab{}.
\newblock \showarticletitle{Scikit-Multiflow: A Multi-Output Streaming
  Framework}.
\newblock \bibinfo{journal}{\emph{J. Mach. Learn. Res.}} \bibinfo{volume}{19},
  \bibinfo{number}{1} (\bibinfo{date}{Jan.} \bibinfo{year}{2018}),
  \bibinfo{pages}{2915–2914}.
\newblock
\showISSN{1532-4435}


\bibitem[\protect\citeauthoryear{Pedregosa, Varoquaux, Gramfort, Michel,
  Thirion, Grisel, Blondel, Prettenhofer, Weiss, Dubourg, Vanderplas, Passos,
  Cournapeau, Brucher, Perrot, and Duchesnay}{Pedregosa et~al\mbox{.}}{2011}]%
        {scikit-learn}
\bibfield{author}{\bibinfo{person}{F. Pedregosa}, \bibinfo{person}{G.
  Varoquaux}, \bibinfo{person}{A. Gramfort}, \bibinfo{person}{V. Michel},
  \bibinfo{person}{B. Thirion}, \bibinfo{person}{O. Grisel},
  \bibinfo{person}{M. Blondel}, \bibinfo{person}{P. Prettenhofer},
  \bibinfo{person}{R. Weiss}, \bibinfo{person}{V. Dubourg}, \bibinfo{person}{J.
  Vanderplas}, \bibinfo{person}{A. Passos}, \bibinfo{person}{D. Cournapeau},
  \bibinfo{person}{M. Brucher}, \bibinfo{person}{M. Perrot}, {and}
  \bibinfo{person}{E. Duchesnay}.} \bibinfo{year}{2011}\natexlab{}.
\newblock \showarticletitle{Scikit-learn: Machine Learning in {P}ython}.
\newblock \bibinfo{journal}{\emph{Journal of Machine Learning Research}}
  \bibinfo{volume}{12} (\bibinfo{year}{2011}), \bibinfo{pages}{2825--2830}.
\newblock


\bibitem[\protect\citeauthoryear{Pietruczuk, Rutkowski, Jaworski, and
  Duda}{Pietruczuk et~al\mbox{.}}{2017}]%
        {pietruczuk2017adjust}
\bibfield{author}{\bibinfo{person}{Lena Pietruczuk}, \bibinfo{person}{Leszek
  Rutkowski}, \bibinfo{person}{Maciej Jaworski}, {and} \bibinfo{person}{Piotr
  Duda}.} \bibinfo{year}{2017}\natexlab{}.
\newblock \showarticletitle{How to adjust an ensemble size in stream data
  mining?}
\newblock \bibinfo{journal}{\emph{Information Sciences}}  \bibinfo{volume}{381}
  (\bibinfo{year}{2017}), \bibinfo{pages}{46--54}.
\newblock


\bibitem[\protect\citeauthoryear{Rana, Gupta, and Venkatesh}{Rana
  et~al\mbox{.}}{2015}]%
        {rana2015differentially}
\bibfield{author}{\bibinfo{person}{Santu Rana}, \bibinfo{person}{Sunil~Kumar
  Gupta}, {and} \bibinfo{person}{Svetha Venkatesh}.}
  \bibinfo{year}{2015}\natexlab{}.
\newblock \showarticletitle{Differentially private random forest with high
  utility}. In \bibinfo{booktitle}{\emph{2015 IEEE International Conference on
  Data Mining}}. IEEE, \bibinfo{pages}{955--960}.
\newblock


\bibitem[\protect\citeauthoryear{Rubinstein, Bartlett, Huang, and
  Taft}{Rubinstein et~al\mbox{.}}{2009}]%
        {rubinstein2009learning}
\bibfield{author}{\bibinfo{person}{Benjamin~IP Rubinstein},
  \bibinfo{person}{Peter~L Bartlett}, \bibinfo{person}{Ling Huang}, {and}
  \bibinfo{person}{Nina Taft}.} \bibinfo{year}{2009}\natexlab{}.
\newblock \showarticletitle{Learning in a large function space:
  Privacy-preserving mechanisms for SVM learning}.
\newblock \bibinfo{journal}{\emph{arXiv preprint arXiv:0911.5708}}
  (\bibinfo{year}{2009}).
\newblock


\bibitem[\protect\citeauthoryear{Sberbank}{Sberbank}{2017}]%
        {sberbankdata}
\bibfield{author}{\bibinfo{person}{Sberbank}.} \bibinfo{year}{2017}\natexlab{}.
\newblock \bibinfo{booktitle}{\emph{Sberbank Russian Housing Market Dataset}}.
\newblock
\urldef\tempurl%
\url{https://www.kaggle.com/c/sberbank-russian-housing-market/data}
\showURL{%
\tempurl}


\bibitem[\protect\citeauthoryear{Scholz and Klinkenberg}{Scholz and
  Klinkenberg}{2007}]%
        {scholz2007boosting}
\bibfield{author}{\bibinfo{person}{Martin Scholz} {and} \bibinfo{person}{Ralf
  Klinkenberg}.} \bibinfo{year}{2007}\natexlab{}.
\newblock \showarticletitle{Boosting classifiers for drifting concepts}.
\newblock \bibinfo{journal}{\emph{Intelligent Data Analysis}}
  \bibinfo{volume}{11}, \bibinfo{number}{1} (\bibinfo{year}{2007}),
  \bibinfo{pages}{3--28}.
\newblock


\bibitem[\protect\citeauthoryear{Shokri and Shmatikov}{Shokri and
  Shmatikov}{2015}]%
        {shokri2015privacy}
\bibfield{author}{\bibinfo{person}{Reza Shokri} {and} \bibinfo{person}{Vitaly
  Shmatikov}.} \bibinfo{year}{2015}\natexlab{}.
\newblock \showarticletitle{Privacy-preserving deep learning}. In
  \bibinfo{booktitle}{\emph{Proceedings of the 22nd ACM SIGSAC conference on
  computer and communications security}}. ACM, \bibinfo{pages}{1310--1321}.
\newblock


\bibitem[\protect\citeauthoryear{Wang, Fan, Yu, and Han}{Wang
  et~al\mbox{.}}{2003}]%
        {wang2003mining}
\bibfield{author}{\bibinfo{person}{Haixun Wang}, \bibinfo{person}{Wei Fan},
  \bibinfo{person}{Philip~S Yu}, {and} \bibinfo{person}{Jiawei Han}.}
  \bibinfo{year}{2003}\natexlab{}.
\newblock \showarticletitle{Mining concept-drifting data streams using ensemble
  classifiers}. In \bibinfo{booktitle}{\emph{Proceedings of the ninth ACM
  SIGKDD international conference on Knowledge discovery and data mining}}.
  \bibinfo{pages}{226--235}.
\newblock


\bibitem[\protect\citeauthoryear{Xiang, Li, Li, Zhong, and Yu}{Xiang
  et~al\mbox{.}}{2018}]%
        {xiang2018collaborative}
\bibfield{author}{\bibinfo{person}{Tao Xiang}, \bibinfo{person}{Yang Li},
  \bibinfo{person}{Xiaoguo Li}, \bibinfo{person}{Shigang Zhong}, {and}
  \bibinfo{person}{Shui Yu}.} \bibinfo{year}{2018}\natexlab{}.
\newblock \showarticletitle{Collaborative ensemble learning under differential
  privacy}. In \bibinfo{booktitle}{\emph{Web Intelligence}},
  Vol.~\bibinfo{volume}{16}. IOS Press, \bibinfo{pages}{73--87}.
\newblock


\bibitem[\protect\citeauthoryear{Yang, Ciptadi, Laziuk, Ahmadzadeh, and
  Wang}{Yang et~al\mbox{.}}{2021}]%
        {yang2021bodmas}
\bibfield{author}{\bibinfo{person}{Limin Yang}, \bibinfo{person}{Arridhana
  Ciptadi}, \bibinfo{person}{Ihar Laziuk}, \bibinfo{person}{Ali Ahmadzadeh},
  {and} \bibinfo{person}{Gang Wang}.} \bibinfo{year}{2021}\natexlab{}.
\newblock \showarticletitle{BODMAS: An Open Dataset for Learning based Temporal
  Analysis of PE Malware}. In \bibinfo{booktitle}{\emph{Proceedings of Deep
  Learning and Security Workshop (DLS), in conjunction with IEEE Symposium on
  Security and Privacy (IEEE SP)}}.
\newblock


\bibitem[\protect\citeauthoryear{Yu and Abraham}{Yu and Abraham}{2017}]%
        {yu2017ddmhierarchical}
\bibfield{author}{\bibinfo{person}{Shujian Yu} {and} \bibinfo{person}{Zubin
  Abraham}.} \bibinfo{year}{2017}\natexlab{}.
\newblock \showarticletitle{Concept drift detection with hierarchical
  hypothesis testing}. In \bibinfo{booktitle}{\emph{Proceedings of the 2017
  SIAM International Conference on Data Mining}}. SIAM,
  \bibinfo{pages}{768--776}.
\newblock


\bibitem[\protect\citeauthoryear{Zhang and Li}{Zhang and Li}{2019}]%
        {zhang2019bayesdp}
\bibfield{author}{\bibinfo{person}{Guozheng Zhang} {and} \bibinfo{person}{Shuyu
  Li}.} \bibinfo{year}{2019}\natexlab{}.
\newblock \showarticletitle{Research on differentially private bayesian
  classification algorithm for data streams}. In \bibinfo{booktitle}{\emph{2019
  IEEE 4th International Conference on Big Data Analytics (ICBDA)}}. IEEE,
  \bibinfo{pages}{14--20}.
\newblock


\end{thebibliography}

\end{document}